\documentclass[twoside]{article}

\usepackage[accepted]{aistats2024}

\usepackage[round,authoryear]{natbib}

\usepackage{bbm}
\usepackage{mathtools}
\usepackage{verbatim}
\usepackage{comment}
\usepackage[linesnumbered, ruled,vlined]{algorithm2e}

\usepackage{wrapfig}

\usepackage{graphicx}
\usepackage{subfigure}
\usepackage{booktabs} %

\usepackage{collectbox}
\usepackage{tcolorbox}
\usepackage{tabularx}
\usepackage{pifont}
\usepackage{makecell}
\usepackage{multirow}

\usepackage{pdflscape}

\usepackage{hyperref}

\usepackage{tikz}

\usepackage{amsthm,amsmath,amssymb, amsfonts}
\usepackage{thmtools}

\newtheorem{theorem}{Theorem}[section]
\newtheorem{corollary}{Corollary}[section]
\newtheorem{lemma}[theorem]{Lemma}

\newtheorem{definition}[theorem]{Definition}

\usepackage{comment}
\usepackage{xcolor}
\usepackage{mathrsfs}
\usepackage{enumitem}
\usepackage{bbm}
\usepackage{mdframed}

\usepackage{pythonhighlight}

\usepackage{prettyref}
\newcommand{\pref}[1]{\prettyref{#1}}

\newcommand{\savehyperref}[2]{\texorpdfstring{\hyperref[#1]{#2}}{#2}}
\newrefformat{eq}{\savehyperref{#1}{\textup{(\ref*{#1})}}}
\newrefformat{eqn}{\savehyperref{#1}{Equation~\ref*{#1}}}
\newrefformat{con}{\savehyperref{#1}{Conjecture~\ref*{#1}}}
\newrefformat{lem}{\savehyperref{#1}{Lemma~\ref*{#1}}}
\newrefformat{def}{\savehyperref{#1}{Definition~\ref*{#1}}}
\newrefformat{line}{\savehyperref{#1}{line~\ref*{#1}}}
\newrefformat{thm}{\savehyperref{#1}{Theorem~\ref*{#1}}}
\newrefformat{corr}{\savehyperref{#1}{Corollary~\ref*{#1}}}
\newrefformat{sec}{\savehyperref{#1}{Section~\ref*{#1}}}
\newrefformat{app}{\savehyperref{#1}{Appendix~\ref*{#1}}}
\newrefformat{ass}{\savehyperref{#1}{Assumption~\ref*{#1}}}
\newrefformat{ex}{\savehyperref{#1}{Example~\ref*{#1}}}
\newrefformat{fig}{\savehyperref{#1}{Figure~\ref*{#1}}}
\newrefformat{tab}{\savehyperref{#1}{Table~\ref*{#1}}}
\newrefformat{alg}{\savehyperref{#1}{Algorithm~\ref*{#1}}}
\newrefformat{rem}{\savehyperref{#1}{Remark~\ref*{#1}}}
\newrefformat{conj}{\savehyperref{#1}{Conjecture~\ref*{#1}}}
\newrefformat{prop}{\savehyperref{#1}{Proposition~\ref*{#1}}}
\newrefformat{proto}{\savehyperref{#1}{Protocol~\ref*{#1}}}
\newrefformat{prob}{\savehyperref{#1}{Problem~\ref*{#1}}}
\newrefformat{claim}{\savehyperref{#1}{Claim~\ref*{#1}}}

\usepackage{c_definitions}

\newmdtheoremenv{test}{Test}

\definecolor{DarkRed}{rgb}{0.75,0,0}
\definecolor{DarkGreen}{rgb}{0,0.5,0}
\definecolor{DarkPurple}{rgb}{0.5,0,0.5}
\definecolor{DarkBlue}{rgb}{0,0,0.7}

\hypersetup{colorlinks=true, plainpages = false, citecolor = DarkBlue, urlcolor = DarkBlue, filecolor = black, linkcolor =DarkBlue}

\newcommand{\reg}{\mathrm{reg}}
\newcommand{\Reg}{\mathrm{Reg}}
\newcommand{\indicator}[1]{\mathbf{1}\{#1\}}

\allowdisplaybreaks

\begin{document}

\twocolumn[

\aistatstitle{Data-Driven Online Model Selection With Regret Guarantees}

\aistatsauthor{ Aldo Pacchiano \And Christoph Dann \And Claudio Gentile}

\aistatsaddress{ Boston University \\ Broad Institute of MIT and Harvard \And  Google Research \And Google Research} ]

\begin{abstract}
\vspace{-0.15in}
We consider model selection for sequential decision making in stochastic environments with bandit feedback, where a meta-learner has at its disposal a pool of base learners, and decides on the fly which action to take based on the policies recommended by each base learner. Model selection is performed by regret balancing but, unlike the recent literature on this subject, we do not assume any prior knowledge about the base learners like candidate regret guarantees; instead, we uncover these quantities in a data-driven manner. The meta-learner is therefore able to leverage the {\em realized} regret incurred by each base learner for the learning environment at hand (as opposed to the {\em expected} regret), and single out the best such regret. 
We design two model selection algorithms operating with this more ambitious notion of regret and, besides proving model selection guarantees via regret balancing, we experimentally demonstrate the compelling practical benefits of dealing with actual regrets instead of candidate regret bounds.
\end{abstract}

\vspace{-0.1in}
\section{INTRODUCTION}\label{sec:intro}
\vspace{-0.1in}
In online model selection for sequential decision making, the learner has access to a set of base learners and the goal is to adapt during learning to the best base learner that is the most suitable for the current environment. The set of base learners typically comes from instantiating different modelling assumptions or hyper-parameter choices, e.g., complexity of the reward model or the $\epsilon$-parameter in $\epsilon$-greedy. Which choice, and therefore which base learner, works best is highly dependent on the problem instance at hand, so that good online model selection solutions are important for robust sequential decision making. This has motivated an extensive study of model selection questions \citep[e.g.,][ and others cited below]{agarwal2017corralling,abbasi2020regret,ghosh2020problem,chatterji2020osom, bibaut2020rate,foster2020adapting,lee2020online, wei2022model} in bandit and reinforcement learning problems.
While some of these works have developed custom solutions for specific model selection settings, for instance, selecting among a nested set of linear policy classes in contextual bandits  \citep[e.g.][]{foster2019model}, the relevant literature also provides several general purpose approaches that work in a wide range of settings. Among the most prominent ones are FTRL-based (follow-the-regularized-leader) algorithms, including EXP4 \citep{odalric2011adaptive}, Corral \citep{agarwal2017corralling, pacchiano2020model} and Tsallis-INF \citep{arora2020corralling}, as well as algorithms based on regret balancing \citep{abbasi2020regret,pacchiano2020regret, pmlr-v139-cutkosky21a, pacchiano2022best}.

These methods usually come with theoretical guarantees of the following form: the \emph{expected regret} (or \emph{high-probability regret}) of the model selection algorithm is not much worse than the expected regret (or high probability regret) of the best base learner. Such results are reasonable and known to be unimprovable in the worst-case \citep{marinov2021pareto}. Yet, it is possible for model selection to achieve expected regret that is systematically smaller than that of any base learner.
This may seem surprising at first, but it can be explained through an example when considering the large variability across individual runs of each base learner on the same environment.

The situation is illustrated in \pref{fig:expected_regret_motivation}. On the left, we plot the cumulative expected regret of two base learners, along with the corresponding behavior of one of our model selection algorithms (ED$^2$RB -- see \pref{sec:ED2RB} below) run on top of them. On the right, we unpack the cumulative expected regret curve of one of the two base learners from the left plot, and display ten independent runs of this base learner on the same environment, together with the resulting expected regret curve (first 1000 rounds only).
Since the model selection algorithm has access to two base learners simultaneously, it can leverage a good run of either of two, and thereby achieve a good run more likely than any base learner individually, leading to overall smaller expected regret.

Such high variability in performance across individual runs of a base learner is indeed fairly common in model selection, for instance when base learners correspond to different hyper-parameters that control the explore-exploit trade-off. For a hyper-parameter setting that explores too little for the given environment, the base learner becomes unreliable and either is lucky and converges quickly to the optimal solution or unlucky and gets stuck in a suboptimal one.
This phenomenon is a key motivation for our work. Instead of model selection methods that merely compete with the expected regret of any base learner, we design model selection solutions that compete with the regret {\em realizations} of any base learner, and have (data-dependent) theoretical guarantees that validate this ability.

\begin{figure*}
    \centering
    \includegraphics[width=0.8\textwidth]{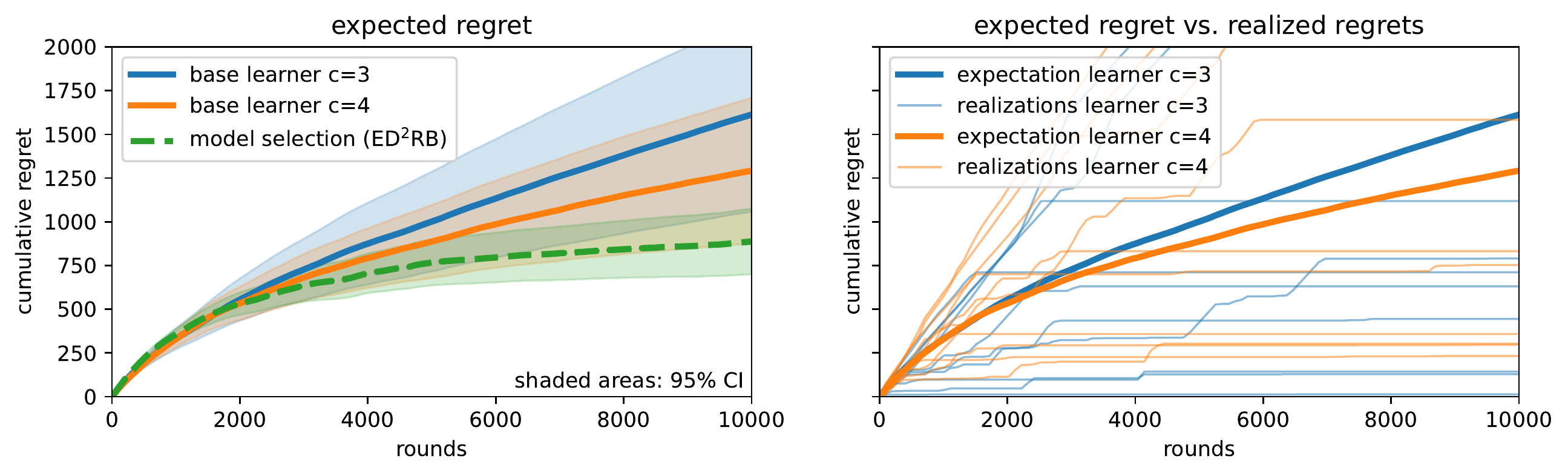}
    \vspace{-0.2in}
    \caption{{\bf Left:} Expected regret of two base learners (UCB on MAB with confidence scaling $c$ controlling explore-exploit trade-off) and a model selection algorithm on top of them. The model selection algorithm has smaller expected regret than any base learner. {\bf Right:} Expected regret and individual regret realizations (independent sample runs) of base learners. The base learners have highly variable performance which model selection can capitalize on. Detailed setup in \pref{app:earlyfigdetails}.}
    \label{fig:expected_regret_motivation}
    \vspace{-0.1in}
\end{figure*}

While the analysis of FTRL-based model selection algorithms naturally lends itself to work with expected regret \citep[e.g.][]{agarwal2017corralling}, the existing guarantees for regret balancing work with realized regret of base learners \citep[e.g.][]{pacchiano2020regret, pmlr-v139-cutkosky21a}. Concretely, regret balancing requires each learner to be associated with a {\em candidate} regret bound, and the model selection algorithm competes with the regret bound of the best among the well-specified learner, those learners whose regret realization is below their candidate bound. Setting a-priori tight candidate regret bounds for base learners is a main limitation for existing regret balancing methods, as the resolution of these bounds is often the one provided by a (typically coarse) theoretical analysis.
As suggested in earlier work, we can create several copies of each base learner with different candidate bounds, but we find this not to perform well in practice due to the high number of resulting base learners. Another point of criticism for existing regret balancing methods is that, up to deactivation of base learners, these methods do not adapt to observations, since their choice among active base learners is determined solely by the candidate regret bounds, which are set a-priori. 

In this work, we address both limitations, and propose two new regret balancing algorithms for model selection with bandit feedback that do not require knowing candidate regret bounds. Instead, the algorithms determine the right regret bounds sequentially in a data-driven manner, allowing them to adapt to the regret realization of the best base learner. We prove this by deriving regret guarantees that share the same form with existing theoretial results, but replace expected regret rates or well-specified regret bounds with realized regret rates, which can be much sharper (as in the example in \pref{fig:expected_regret_motivation}).
From an empirical standpoint, we illustrate the validity of our approach by carrying out an experimental comparison with competing approaches to model selection via base learner pooling, and find that our new algorithms systematically outperform the tested baselines.

\section{SETUP AND NOTATION}\label{s:setup}
\vspace{-0.1in}
We consider a general sequential decision making framework that covers many important problem classes such as multi-armed bandits, contextual bandits and tabular reinforcement learning as special cases.
This framework or variations of it has been commonly used in model selection \citep[e.g.][]{pmlr-v139-cutkosky21a,wei2022model,pacchiano2022best}.

The learner operates with a policy class $\Pi$ and a set of contexts $\Xcal$ over which is defined a probability distribution $\Dcal$, unknown to the learner.
In bandit settings, each policy $\pi$ is a mapping from contexts $\Xcal$ to $\Delta_{\mathcal{A}}$, where $\mathcal{A}$ is an action space and $\Delta_{\mathcal{A}}$ denotes the set of probability distributions over ${\mathcal A}$. However, the concrete form of $\Pi$, $\Xcal$ or $\Acal$ is not relevant for our purposes.
We only need that each policy $\pi \in \Pi$  is associated with a fixed expected reward mapping $\mu^\pi \colon \Xcal \rightarrow [0, 1]$ of the form $\mu^{\pi}(x) = \EE[r | x, \pi]$, 
which is unknown to the learner.
In each round $t \in \NN$ of the sequential decision process, the learner first decides on a policy $\pi_t \in \Pi$. The environment then draws a context $x_t \sim \Dcal$ 
as well as a reward observation $r_t \in [0, 1]$ such that 
$\EE[r_t | x_t, \pi_t] = \mu^{\pi_t}(x_t)$. The learner receives $(x_t, r_t)$ before the next round starts.

We call $v^\pi = \EE_{x \sim \Dcal}[\mu^{\pi}(x)]$ the \emph{value} of a policy $\pi \in \Pi$ and define the instantaneous regret of $\pi$ as
\begin{align}
    \reg(\pi) = v^\star - v^\pi = \EE_{x \sim \Dcal}[\mu^{\pi_\star}(x) - \mu^{\pi}(x)]\label{e:reg}
\end{align}
where $\pi_\star \in \argmax_{\pi \in \Pi} v^\pi$ is an optimal policy and $v^\star$ its value. The total regret after $T$ rounds of an algorithm that chooses policies $\pi_1, \pi_2, \dots$ is
$\Reg(T) = \sum_{t=1}^T \reg(\pi_t)$.
Note that $\Reg(T)$ is a random quantity since the policies $\pi_t$ selected by the algorithm depend on past observations, which are themselves random variables. Yet, we use in (\ref{e:reg}) a pseudo-regret notion that takes expectation over reward realizations and context draws. This is most convenient for our purposes but we can achieve guarantees without those expectations by paying an additive $O(\sqrt{T})$ term, as is standard. We also use $u_T = \sum_{t=1}^T v^{\pi_t}$ for the total value accumulated by the algorithm over the $T$ rounds.

\vspace{-0.1in}
\paragraph{Base learners.}
The learner (henceforth called {\em meta-learner}) is in turn given access to $M$ base learners that the meta-learner can consult when determining the current policy to deploy. Specifically, in each round $t$, the meta-learner chooses one base learner $i_t \in [M] = \{1,\ldots, M\}$ to follow and plays the policy suggested by this base learner. The policy that base learner $i$ recommends in round $t$ is denoted by $\pi^i_t$ and thus $\pi_t = \pi^{i_t}_t$. 
We shall assume that each base learner has an internal state (and internal clock) that gets updated only on the rounds where that base learner is chosen. After being selected in round $t$, base learner $i_t$ will receive from the meta-learner the observation $(x_t,r_t)$.
We use $n^i_t = \sum_{\ell = 1}^{t} \indicator{i_t = i}$ to denote the number of times base learner $i$ happens to be chosen up to round $t$, and by
$u_t^i = \sum_{\ell = 1}^{t} \indicator{i_t = i} v^{\pi^i_t}$ the total value accumulated by base learner $i$ up to this point.
It is sometimes more convenient to use a base learner's internal clock instead of the total round index $t$. To do so, we will use subscripts $(k)$ with parentheses to denote the internal time index of a specific base learner, while subscripts $t$ refer to global round indices. For example, given the sequence of realizatons $(x_1,r_1), (x_2,r_2), \ldots$, $\pi^i_{(k)}$ is the policy base learner $i$ wants to play when being chosen the $k$-th time,
i.e., $\pi^i_t = \pi^i_{(n^i_t)}$.
The total regret incurred by a meta-learner that picks base learners $i_1,\ldots, i_T$ can then be decomposed into the sum of regrets incurred by each base learner:
\vspace{-0.1in}
\begin{align*}
    \Reg(T) = \sum_{t=1}^T \reg(\pi_t) = \sum_{i = 1}^M \sum_{k = 1}^{n^i_T} \reg(\pi^i_{(k)}).
\end{align*}

\subsection{Data-Driven Model Selection}
\vspace{-0.1in}
Our goal is to perform model selection in this setting: We devise sequential decision making algorithms 
that have access to base learners as subroutines and are guaranteed to have regret that is comparable 
to the smallest {\em realized} regret, among all base learners in the pool, despite not knowing a-priori which base learner will happen to be best for the environment at hand ($\Dcal$ and $\mu^\pi$), and the actual realizations $(x_1,r_1), (x_2,r_2), \ldots, (x_{T},r_{T})$.

In order to better quantify this notion of realized regret, the following definition will come handy.
\begin{definition}[regret scale and coefficients]\label{def:regretcoeff}
The \emph{regret scale} of base learner $i$ after being played $k$ rounds is $\frac{\sum_{\ell=1}^{k} \reg(\pi_{(\ell)}^i)}{\sqrt{k}}$.
For a positive constant $d_{\min}$, the \emph{regret coefficient} of base learner $i$ after being played $k$ rounds is defined as
\vspace{-0.12in}
\begin{align*}
    d^i_{(k)} = \max \,\,\,\Biggl\{ \frac{\sum_{\ell=1}^{k} \reg(\pi_{(\ell)}^i)}{\sqrt{k}}, d_{\min} \Biggl\}.
\end{align*}
That is, $d^i_{(k)} \geq d_{\min}$ is the smallest number such that the incurred regret is bounded as $\sum_{\ell=1}^{k} \reg(\pi^i_{(\ell)}) \leq d^i_{(k)} \sqrt{k}$.
Further we define the \emph{monotonic regret coefficient} of base learner $i$ after being played $k$ rounds as
$\bar d^i_{(k)} = \max_{\ell \in [k]} d^i_{(\ell)}$.
\end{definition}
We use a $\sqrt{k}$ rate in this definition since that is the most commonly targeted regret rate in stochastic settings. Our results can be adapted, similarly to prior work \citep{pacchiano2020regret} to other rates but the $\sqrt{T}$ barrier for model selection \citep{pacchiano2020model} remains of course.

It is worth emphasizing that both $d^i_{(k)}$ and $\bar d^i_{(k)}$ in the \pref{def:regretcoeff} are random variables depending on $(x_1,r_1), (x_2,r_2), \ldots, (x_\ell,r_\ell)$, where $\ell = \min\{t\,:\, n^i_t = k\}$. We illustrate them in \pref{fig:regret_coefficients}.

\begin{figure*}
  \centering
  \includegraphics[width=0.77\textwidth]{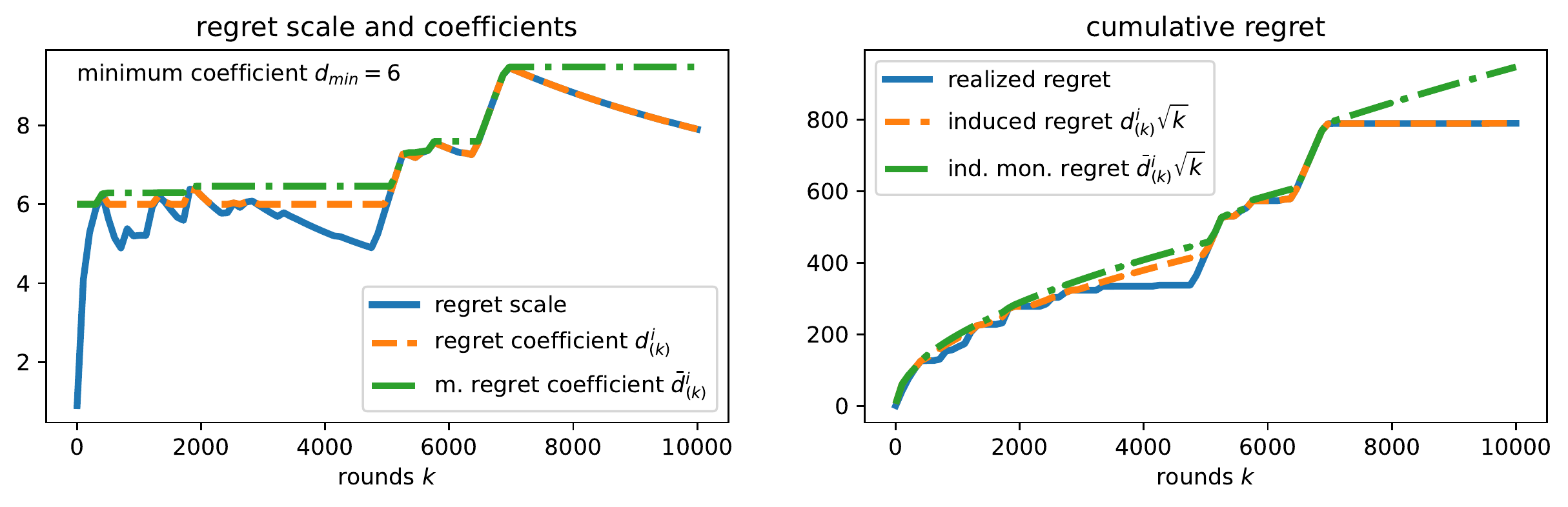}
  \vspace{-0.2in}
  \caption{Illustration of \pref{def:regretcoeff} for one of the baseline realizations from \pref{fig:expected_regret_motivation}.
  \textbf{Left:} Evolution of regret scale, coefficient and monotonic coefficient.
  \textbf{Right:} The same curves multiplied by $\sqrt{k}$. The induced regret bounds from regret coefficients follow the realized regret closely, the non-monotonic version more closely than the monotonic.}
  \label{fig:regret_coefficients}
  \vspace{-0.1in}
\end{figure*}

\subsection{Running Examples}
\vspace{-0.1in}
\label{sec:running_example}
The above formalization encompasses a number of well-known online learning frameworks, including finite horizon Markov decision processes and contextual bandits, and model selection questions therein. We now introduce two examples but refer to earlier works on model selection for a more exhaustive list \citep[e.g.][]{pmlr-v139-cutkosky21a,wei2022model,pacchiano2022best}.

\noindent{\bf Tuning UCB exploration coefficient in multi-armed-bandits.} As a simple illustrative example, we consider multi-armed bandits where the learner chooses in each round an action $a_t$ from a finite action set $\Acal$ and receives a reward $r_t$ drawn from a distribution with mean $\mu^{a_t}$ and unknown but bounded variance $\sigma^2$. In this setting, we directly identify each policy with an action, i.e., $\Pi = \Acal$ and define the context $\Xcal = \{ \varnothing \}$ as empty. The value of an action / policy $a$ is simply $v^a = \mu^a$.

The variance $\sigma$ strongly affects the amount of exploration necessary, thereby controlling the difficulty or ``complexity'' of the learning task. Since the explore-exploit of a learner is typically controlled through a hyper-parameter, it is beneficial to perform model selection among base learners with different trade-offs to adapt to the right complexity of the environment at hand.
We use a simple UCB strategy as a base learner that chooses the next action as $\argmax_{a \in \Acal} \hat \mu(a) + c \sqrt{\frac{\ln(n(a) / \delta)}{n(a)}}$ where $n(a)$ and $\hat \mu(a)$ are the number of pulls of arm $a$ so far and the average reward observed. Here $c$ is the confidence scaling and we instantiate different base learners $i \in [M]$ with different choices $c_1, \dots, c_M$ for $c$. The goal is to adapt to the best confidence scaling $c_{i_\star}$, without knowing the true variance $\sigma^2$.\footnote{We choose this example for its simplicity. An alternative without model selection would be UCB with empirical Bernstein confidence bounds \citep{audibert2007tuning}. However, adaptation with model selection works just as well in more complex settings e.g. linear bandits and MDP, where empirical variance confidence bounds are not available or much more complicated.}

\noindent{\bf Nested linear bandits.} In the stochastic linear bandit model, the learner chooses an action $a_t \in \Acal$ from a large but finite action set $\mathcal{A} \subset \mathbb{R}^{d}$, for some dimension $d>0$ and receives as reward $r_t = a_t^\top \omega$ + white noise, where $\omega \in \mathbb{R}^{d}$ is a fixed but unknown reward vector.
This fits in our framework by considering policies of the form $\pi_\theta(x) = \argmax_{a \in \Acal} \langle a, \theta\rangle$ for a parameter $\theta \in \RR^d$, defining contexts $\Xcal = \{ \varnothing\}$ as empty and the mean reward as $\mu^{\pi}(x) = \pi(x)^\top \omega$, which is also the value $v^{\pi}$.

We here consider the following model selection problem, that was also a motivating application in \citet{pmlr-v139-cutkosky21a}. The action set $\mathcal{A} \subset \mathbb{R}^{d_M}$ has some maximal dimension $d_M>0$, and we have an increasing sequence of $M$ dimensions $d^1 < \ldots < d^M$. Associated with each $d^i$ is a base learner that only considers policies $\Pi_i$ of the form $\pi_{\theta_i}(x) = \argmax_{a \in \Acal} \langle P_{d_i}[a], \theta_i \rangle$ for $\theta_i \in \RR^{d^i}$ and $P_{d^i}[\cdot]$ being the projection onto the first $d^i$ dimensions. That is, the $i$-th base learner operates only on the first $d^{i}$ components of the unknown reward vector $\omega \in \mathbb{R}^{d^M}$. If we stipulate that only the first $d^{i_\star}$ dimensions of $\omega \in \RR^{d^M}$ are non-zero ($d^{i_\star}$ being unknown to the learner) we are in fact competing in a regret sense against the base learner that operates with the policy class $\Pi_{i_\star}$, the one at the ``right" level of complexity for the underlying $\omega$.

\noindent{\bf Nested stochastic linear contextual bandits.} We also consider a contextual version of the previous setting \citep[Ch.~19]{lattimore2020bandit} where where context $x_t \in \Xcal$ are drawn i.i.d. and which a policy maps to some action $a_t \in \Acal$. The expected reward is then $\mu^\pi(x) = \psi(x,\pi(x))^\top \omega$ for a known feature embedding $\psi\,:\,\Xcal\times\Acal \rightarrow \mathcal{R}^d$, and an unknown vector $\omega \in \mathcal{R}^d$. Just as above, we consider the nested version of this setting where $\psi$ and $\omega$ live in a large ambient dimension $d^{M}$ but only the first $d^{i_\star}$ entries of $\omega$ are non-zero.

\vspace{-0.1in}
\section{DATA-DRIVEN REGRET BALANCING}\label{sec:balancingandelimination}
\vspace{-0.1in}
We introduce and analyze two data-driven regret balancing algorithms, which are both shown in \pref{alg:balancing}. Both algorithms maintain over time three main estimators for each base learner: (1) regret coefficients $\widehat d^{i}_{t}$, meant to estimate the monotonic regret coefficients $\bar d^i_{t}$ from \pref{def:regretcoeff}, (2) the average reward estimators $\widehat u^{i}_t/n^i_t$, and (3) the balancing potentials $\phi^i_t$, which are instrumental in the implementation of the exploration strategy based on regret balancing.
At each round $t$ the meta-algorithm picks the base learner $i_t$ with the smallest balancing potential so far (ties broken arbitrarily). The algorithm plays the policy $\pi_t$ suggested by that base learner on the current context $x_t$, receives the associated reward $r_t$, and forwards $(x_t,r_t)$ back to that base learner only.

Where our two meta-learners differ is how they update the regret coefficient $\widehat d^{i_t}_{t}$ of the chosen learner and its potential $\phi^{i_t}_t$. We now introduce each version and the regret guarantee we prove for it.

\begin{algorithm*}[t]
\textbf{Input:} $M$ base learners, minimum regret coefficient $d_{\min}$, failure probability $\delta$\\
Initialize balancing potentials $\phi_1^i = d_{\min}$ and regret coefficient $\widehat d_0^i = d_{\min}$, for all $i \in [M]$ \\
Initialize counts $n^i_0 = 0$, and total value $\widehat u^i_0 = 0$, for all $i \in [M]$\\
\For{rounds $t=1, 2, 3, \dots$}{

Receive context $x_t$\\
Pick base learner $i_t$ with smallest balancing potential: $i_t \in \argmin_{i \in [M]} \phi^i_{t}$\\
Pass $x_t$ to base learner $i_t$\\
Play policy $\pi_t = \pi_t^{i_t}$ suggested by base learner $i_t$ on $x_t$ and receive reward $r_t$\\
 Set $n^i_t = n^i_{t-1}$, $\widehat u^i_t = \widehat u^i_{t-1}$, $\widehat d^i_{t} = \widehat d^i_{t-1}$, and $\phi^i_{t+1} = \phi^i_t$, for $i \in [M] \setminus \{i_t\}$\\
 Update statistics $n^{i_t}_t = n^{i_t}_{t-1} + 1$ and $\widehat u^{i_t}_t = \widehat u^{i_t}_{t-1} +  r_t$\\[2mm]
 \begin{tcolorbox}[sidebyside,lefthand ratio=0.49,title={\hspace{3.3cm}D$^3$RB \hspace{3cm} or \hspace{3cm} ED$^2$RB}, width=0.97\textwidth]%
Perform misspecification test
\begin{align*}
\frac{\widehat u^{i_t}_t}{n_t^{i_t}} + \frac{\widehat d^{i_{t}}_{t-1} \sqrt{n^{i_t}_t}}{n_t^{i_t}} &+ c\sqrt{\frac{\ln \frac{M\ln n^{i_t}_t}{\delta}}{n_t^{i_t}}} 
\\
\vspace*{-5mm}
&< \max_{j \in [M]} \frac{\widehat u^{j}_t}{n_t^{j}} - c\sqrt{\frac{\ln \frac{M\ln n^{j}_t}{\delta}}{n_t^{j}}}
\end{align*}
If test triggered double regret coefficient $\widehat d^{i_t}_{t} =  2 \widehat d^{i_t}_{t-1}$ and otherwise set  $\widehat d^{i_t}_t = \widehat d^{i_t}_{t-1}$\\
Update balancing potential
$\phi^{i_t}_{t+1} = \widehat d^{i_t}_{t} \sqrt{n^{i_t}_t}$
\tcblower
Estimate active regret coefficient
\begin{align*}
\displaystyle  \widehat d^{i_t}_{t} = \max\Biggl\{d_{\min}, \,\,\sqrt{n_t^{i_t}} \Biggl( &\max_{j \in [M]} \frac{\widehat u^{j}_t}{n_t^{j}}
- c\sqrt{\frac{\ln \frac{M\ln n^{j}_t}{\delta}}{n_t^{j}}}\\
&- \frac{\widehat u^{i_t}_t}{n_t^{i_t}} - c\sqrt{\frac{\ln \frac{M\ln n^{i_t}_t}{\delta}}{n_t^{i_t}}}\Biggl)\Biggl\}
\end{align*}
Update balancing potential
\[
\phi^{i_t}_{t+1} = \operatorname{clip}\left( \widehat d^{i_t}_{t} \sqrt{n_t^{i_t}}  ;\,\,\, \phi^{i_t}_t,\,\, 2\phi^{i_t}_t \right)
\]
\end{tcolorbox}

}
\caption{Data Driven Regret Balancing Algorithms (D$^3$RB and ED$^2$RB)}
\label{alg:balancing}
\end{algorithm*}

\vspace{-0.1in}
\subsection{%
Balancing Through Doubling}
\vspace{-0.1in}
Our first meta-algorithm (Doubling Data Driven Regret Balancing  or D$^3$RB) is shown on the left in \pref{alg:balancing}.
Similar to existing regret balancing approaches~\citep{pacchiano2020model, pacchiano2022best}, D$^3$RB performs a misspecification test which checks whether the current estimate of the regret of base learner $i_t$ is compatible with the data collected so far. 
The test compared the average reward $\frac{\widehat u^{i_t}_t}{n_t^{i_t}}$ of the chosen learner against the highest average reward among all learners $\max_{j \in [M]} \frac{\widehat u^{j}_t}{n_t^{j}}$. If the difference is larger than the current regret coefficient $\frac{\widehat d^{i_{t}}_{t-1} \sqrt{n^{i_t}_t}}{n_t^{i_t}}$ permits (accounting for estimation errors by considering appropriate concentration terms), then the we know that $\widehat d^{i_{t}}$ is too small to accurately represent the regret of learner $i_t$ and we double it. 
This deviates from prior regret balancing approaches~\citep{pacchiano2020model,pmlr-v139-cutkosky21a} that simply eliminate a base learner if the misspecification test fails for a given candidate regret bound. Finally, D$^3$RB sets the potential $\phi^{i_t}_t$ as $\widehat d^{i_t}_{t} \sqrt{n^{i_t}_t}$ so that the potential represents an upper-bound on the regret incurred by $i_t$.

Our doubling approach for $\wh d^{i_t}_t$ is algorithmically simple but creates main technical hurdles compared to existing elimination approaches since we have to show that the regret coefficients are adapted fast enough to be accurate and do not introduce undesirable scalings in our upper bounds. By overcoming these hurdles in our analysis, we show the following result quantifies the regret properties of D$^3$RB in terms of the {\em monotonic} regret coefficients of the base learners at hand. 

\begin{restatable}{theorem}{maindouble}\label{thm:maindouble}
With probability at least $1 - \delta$, the regret of D$^3$RB (\pref{alg:balancing}, left) with parameters $\delta$ and $d_{\min} \geq 1$ is bounded in all rounds $T \in \NN$ as\footnote
{
Here and throughout, $\tilde O$ hides log-factors.
}
\begin{align*}
    \Reg(T) = \tilde O \left( \bar d^\star_T M\sqrt{T} + (\bar d^\star_T)^2 \sqrt{MT}\right)
\end{align*}
where $\bar d^\star_T = \min_{i \in [M]} \bar d^i_T = \min_{i \in [M]} \max_{t \in [T]} d^i_t$ is the smallest monotonic regret coefficient among all learners (see \pref{def:regretcoeff}).
\end{restatable}
We defer a discussion of this regret bound and comparison to existing results to \pref{sec:bound_discussion}.

\vspace{-0.05in}
\subsection{Balancing Through Estimation}\label{sec:ED2RB}
\vspace{-0.1in}
While D$^3$RB retains the misspecification test of existing regret balancing approaches, our second algorithm, Estimating Data-Driven Regret Balancing or ED$^2$RB, takes a more direct approach. It estimates the regret coefficient (see right in \pref{alg:balancing}) directly as the highest difference in average reward $\max_{j \in [M]} \frac{\widehat u^{j}_t}{n_t^{j}} - \frac{\widehat u^{i_t}_t}{n_t^{i_t}}$ between $i_t$ and any other learner scaled by $\sqrt{n_t^{i_t}}$, the number of times $i_t$ has been played. Again, we include appropriate concentration terms to account for estimation errors as well as a lower bound of $d_{\min}$ to ensure stability. 

Since this direct estimation approach is more adaptive compared to the doubling approach, it requires a finer analysis to show that the changes in the estimator does not interfere with the balancing property of the algorithm. To ensure the necessary stability, we use a clipped version in ED$^2$RB of the potential of D$^3$RB. The function $\operatorname{clip}(x; a,b)$ therein clips the real argument $x$ to the interval $[a,b]$, 
and makes the potential non-decreasing and not increasing too quickly.

This more careful definition of the balancing potentials
allows us to replace in the regret bound the monotonic regret coefficient $\bar d^\star_T$ with the sharper regret coefficient $d^\star_T$ in the regret guarantee for ED$^2$RB:
\begin{restatable}{theorem}{mainestimate}\label{thm:mainestimate}
With probability at least $1 - \delta$, the regret of ED$^2$RB (\pref{alg:balancing}, right) with parameters $\delta$ and $d_{\min} \geq 1$ is bounded in all rounds $T \in \NN$ as
\begin{align*}
    \Reg(T) = \tilde O \left(d^\star_T M\sqrt{T} + (d^\star_T)^2 \sqrt{MT}\right)
\end{align*}
where $d^\star_T = \min_{i \in [M]} \max_{j \in [M]} d^i_{T_j}$ is the smallest regret coefficient among all learners, and $T_j$ is the last time $t$ when base learner $j$ was played and $\phi^j_{t+1} < 2\phi^j_t$.
\end{restatable}

\vspace{-0.09in}
\subsection{Discussion, Comparison to the Literature}
\label{sec:bound_discussion}
\vspace{-0.1in}
One way to interpret \pref{thm:maindouble} is the following. If the meta-learner were given ahead of time the index of the base learner achieving the smallest monotonic regret coefficient $\bar d^\star_T$, then the meta-learner would follow that base learner from beginning to end. The resulting regret bound for the meta-learner would be of the form
$(\bar d^\star_T)\sqrt{T}$.
Then the price D$^3$RB pays for aggregating the $M$ base learners is essentially a multiplicative factor of the form $M + \bar d^\star_T \sqrt{M}$.

Up to the difference between $d^\star_T$ and $\bar d^\star_T$, the guarantees in \pref{thm:maindouble} and \pref{thm:mainestimate} are identical. Further, since $d^\star_T \leq \bar d^\star_T$, the guarantee for ED$^2$RB is never worse than that for D$^3$RB. It can however be sharper, e.g., in environments with favorable gaps where we expect that a good base learner may achieve a $O(\log(T))$ regret instead of a $\sqrt{T}$ rate and thus $d^i_t$ of that learner would decrease with time. The regret coefficient $d^\star_T$ can benefit from this while $\bar d^\star_T$ cannot decrease with $T$, and thus provide a worse guarantee.

Both D$^3$RB and ED$^2$RB rely on a user-specified parameter $d_{\min}$. In terms of regret coefficients, the regret bounds of the two algorithms have the general form $(d^\star_T)^2/d_{\min} + d_{\min}$. So, if we knew beforehand something about $d_T^\star$, we could set $d_{\min} =  d_T^\star$, and get a linear dependence on $d_T^\star$, otherwise we can always set as default $d_{\min} = 1$ (as we did in \pref{thm:maindouble} and \pref{thm:mainestimate}).
Both our data-dependent guarantees recover existing data-{\em independent} results up to the precise $M$ dependency. Specifically, ignoring $M$ factors, our bounds scale at most as $(d^\star_T)^2 \sqrt{T}$, while the previous literature on the subject (e.g., \cite{pmlr-v139-cutkosky21a}, Corollary 2) scales as $(d^{i_\star})^2\sqrt{T}$. 
We recall that in the data-independent case, regret {\em lower} bounds are contained (in a somewhat implicit form) in
\cite{marinov2021pareto}, where it is shown that one cannot hope in general to achieve better results in terms of $d^{i_\star}$ and $T$, like in particular a 
$d^{i_\star}\sqrt{T}$ regret bound. Only a $(d^{i_\star})^2\sqrt{T}$-like regret guarantee is generally possible.

These prior model selection results based on regret balancing require candidate regret bounds to be specified ahead of time. %
Hence, the corresponding algorithms cannot leverage the favorable cases that our data-dependent bounds automatically adapt to. In particular, in \cite{pmlr-v139-cutkosky21a}, the optimal parameter $d^{i_\star}$ is the smallest candidate regret rate that is larger than the true rate of the optimal (well-specified) base learner. Instead, we do not assume the availability of such candidate regret rates, and our $d^\star_T$ is the {\em true} regret rate of the optimal base learner. In short, \cite{pmlr-v139-cutkosky21a}'s results are competitive with ours only when the above candidate regret rate happens to be very accurate for the best base learner, but this is a fairly strong assumption. Although theoretical regret bounds for base learners can often guide the guess for the regret rate, the values one obtains from those analyses are typically much larger than the true regret rate, as theoretical regret bounds are usually loose by large constants.

So, the goal here is to improve over more traditional data-independent bounds when data is benign or typical. Observe that in practice $d^\star_T$ can also be {\em decreasing} with $T$ (as we will show multiple times in our experiments in Section \ref{sec:experiments}). One such relevant case is when the individual base learner runs have large variances (recall Figure \ref{fig:expected_regret_motivation}). 

From a technical standpoint, we do indeed build on the existing technique for analyzing regret balancing by \citet{pacchiano2020model,pmlr-v139-cutkosky21a}. Yet, their  analysis heavily relies on fixed candidate regret bounds, and removing those introduces several technical challenges, like disentangling the balancing potentials $\phi^i_t$ from the estimated regret coefficients $\widehat d^i_t$, and combine with clipping or the doubling estimator. This allows us to show the necessary invariance properties that unlocks our improved data-dependent guarantees. See \pref{app:doubling_proofs} and \ref{app:estimating_proofs}.

Departing from regret balancing techniques, model selection can also revolve around Follow-The-Regularized Leader-like schemes (e.g., \citep{agarwal2017corralling, pacchiano2020model, arora2020corralling}). However, even in those papers, $d^{i_\star}$ is the expected regret scale, thus never sharper than our $d^\star_T$, and also not able to capture favorable realizations. As we shall see in Section \ref{sec:experiments}, there is often a stark difference between the expected performance and the data-dependent performance, which confirms that the improvement in our bounds is important in practice.

\begin{figure*}
\includegraphics[width=0.338\textwidth]{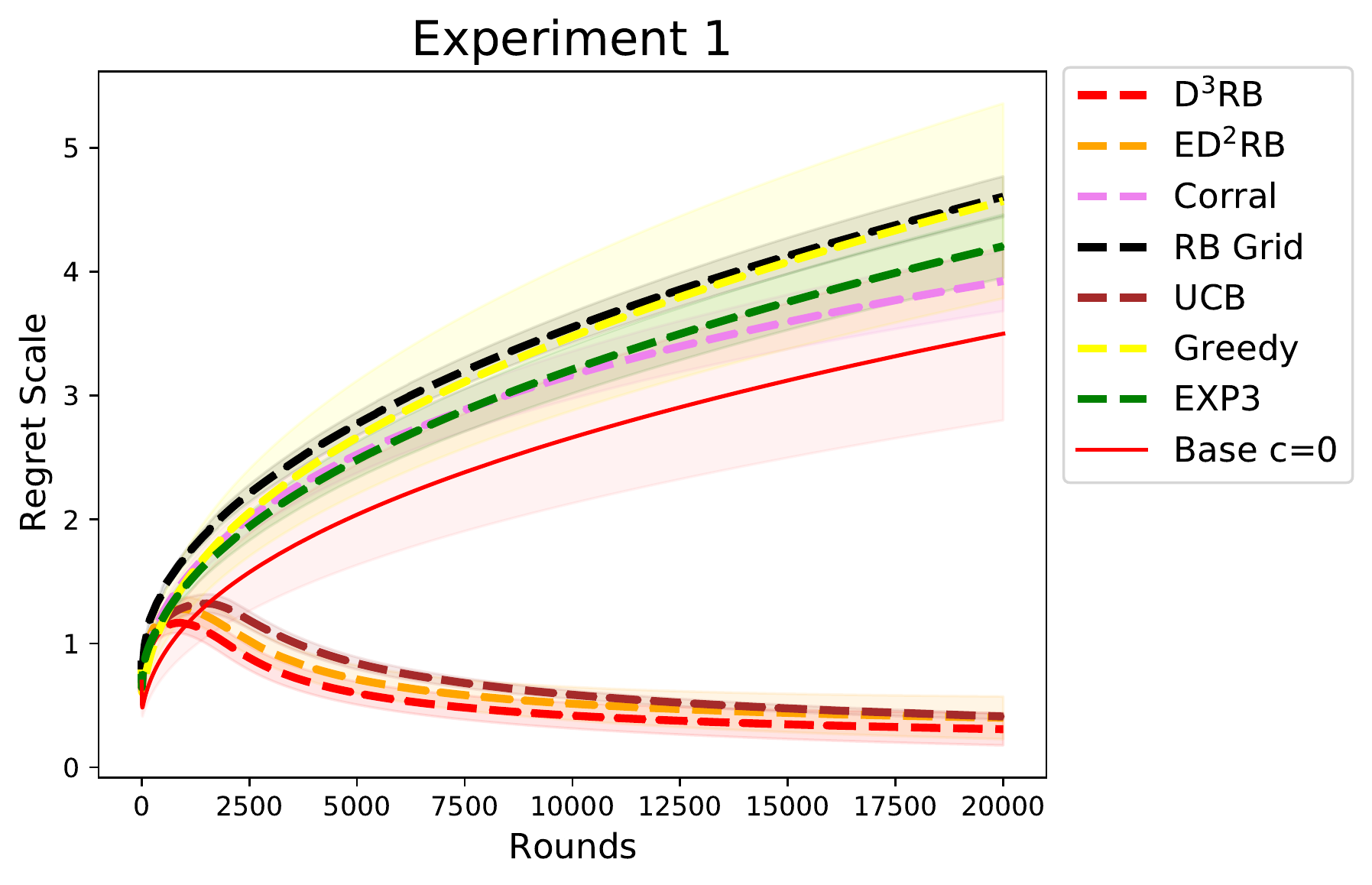}
\hspace{-3mm}
   \includegraphics[width=0.338\textwidth]{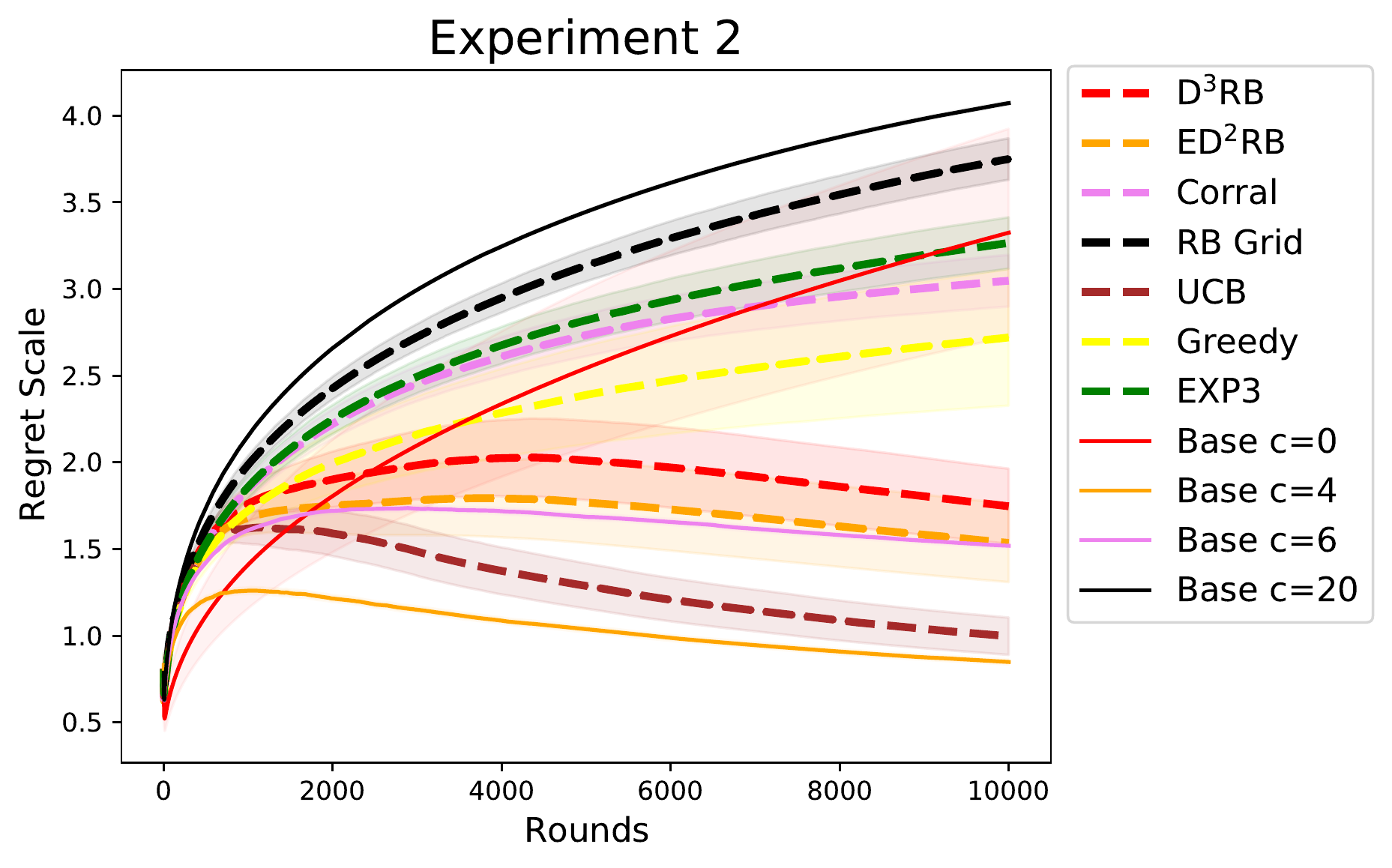}
\hspace{-3mm}
   \includegraphics[width=0.338\textwidth]{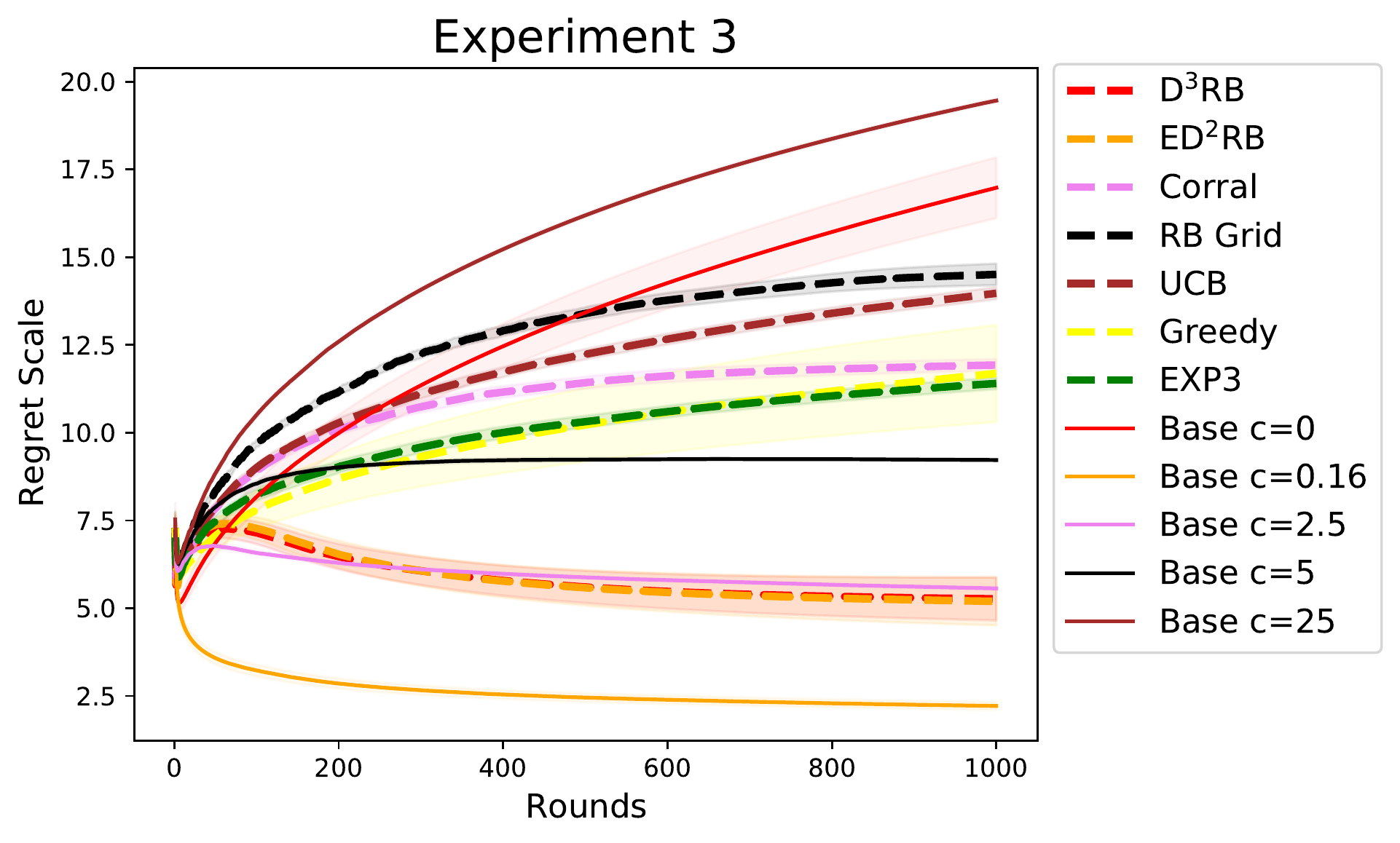}\\
    \includegraphics[width=0.338\textwidth]{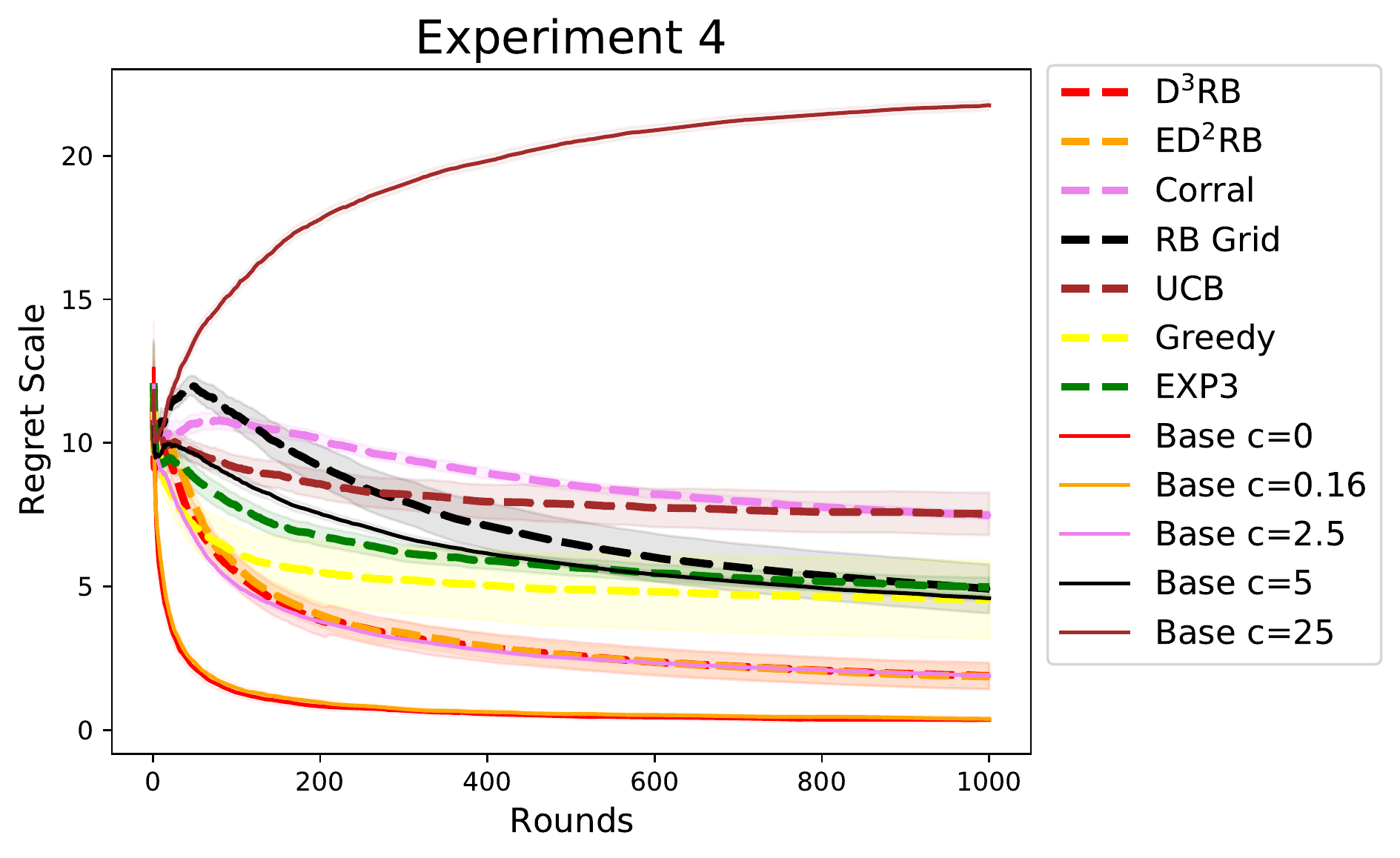}\hspace{-3mm}
     \includegraphics[width=0.338\textwidth]{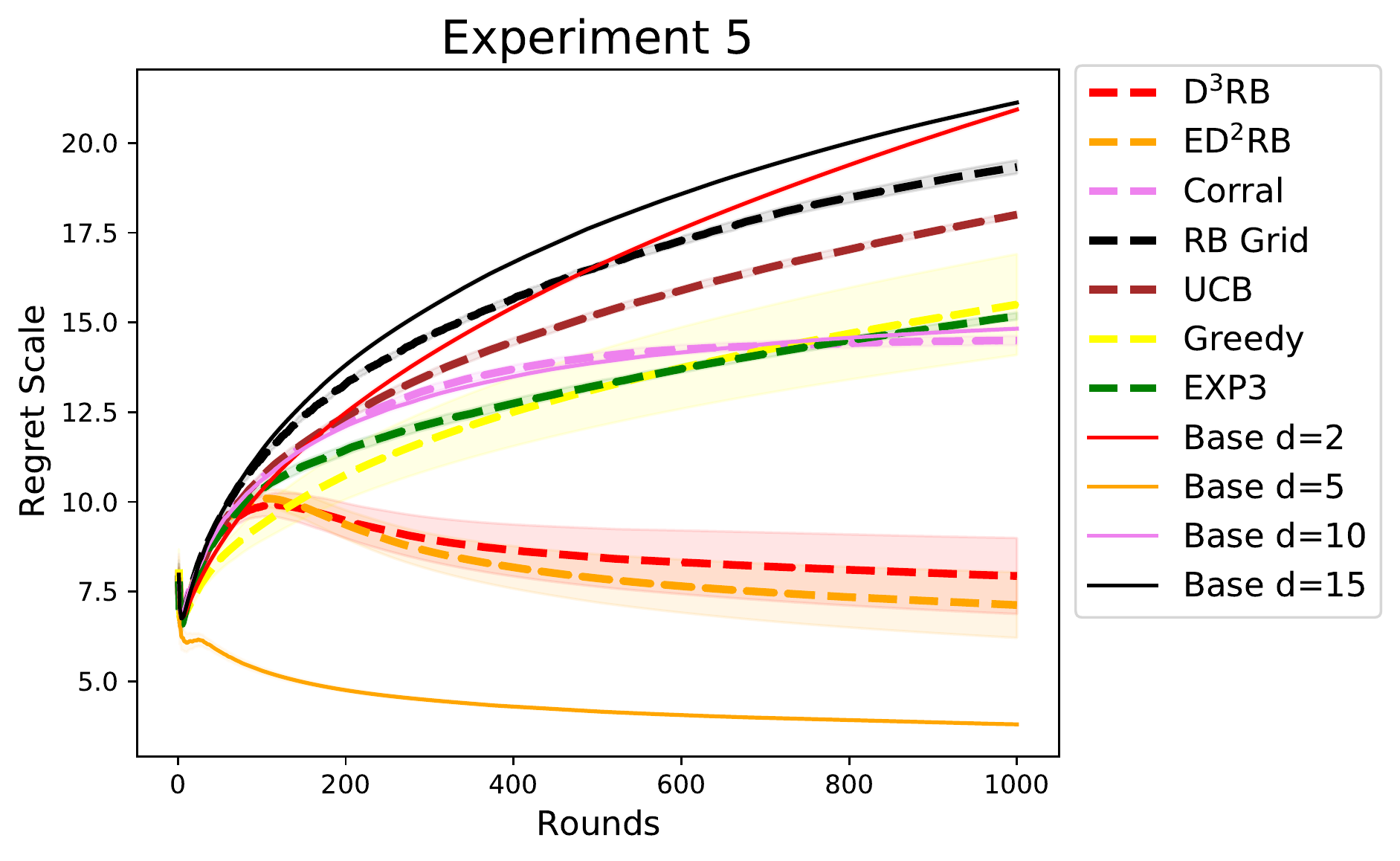}\hspace{-3mm}
    \includegraphics[width=0.338\textwidth]{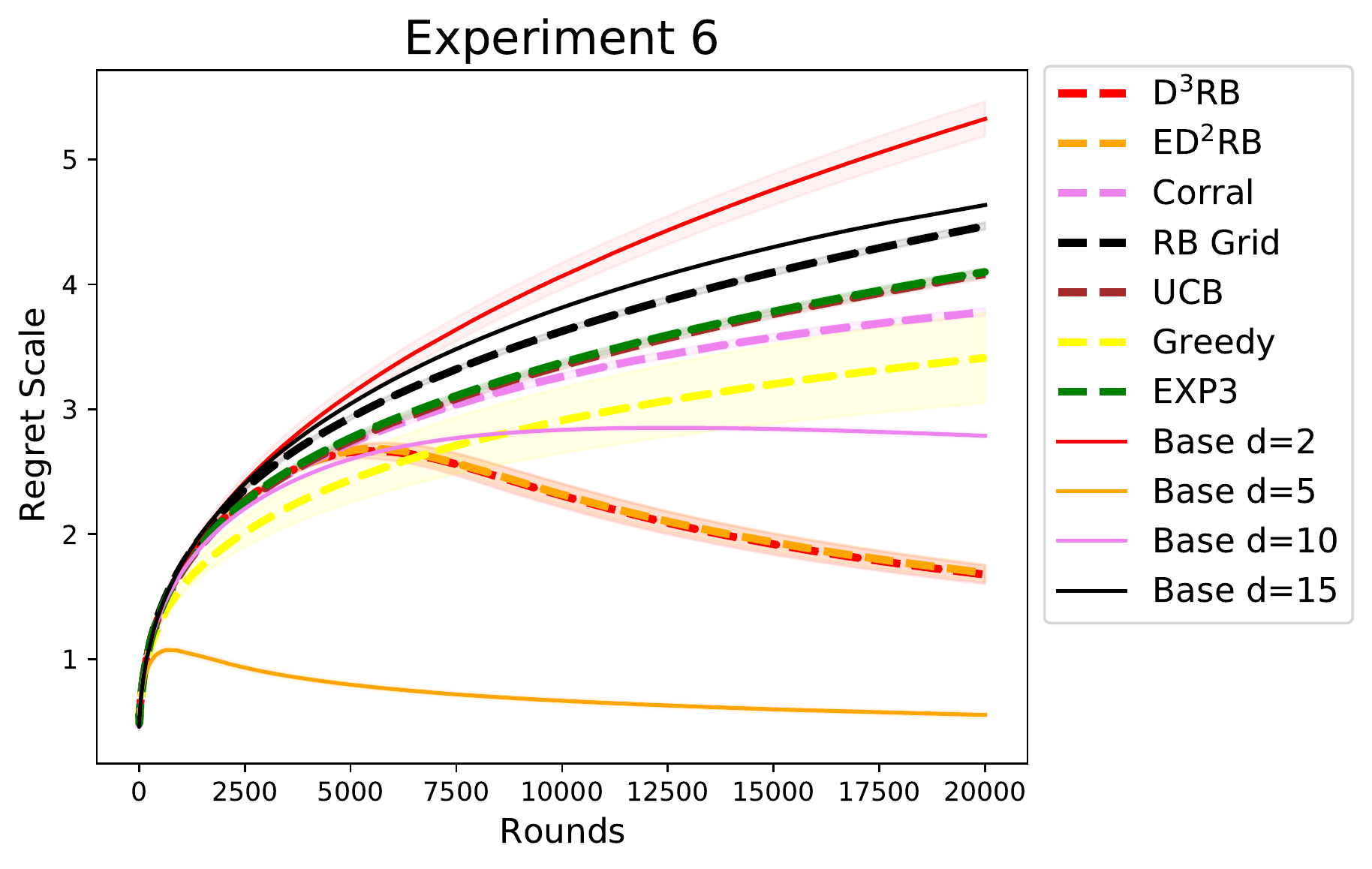}
  \vspace{-0.1in}
  \caption{Average performance comparing all meta-learners (see \pref{tab:general_overview} for reference). \textbf{Experiment $1$:}. Self model selection. See also Figure \ref{fig:expected_regret_sample_runs_appendix} in Appendix \ref{app:detailexperiments}, containing regret curves for \textbf{D$^3$RB} and \textbf{ED$^2$RB} on a single realization. \textbf{Experiment $2$:} base learners ( \textbf{UCB}) with different confidence multipliers $c$. \textbf{Experiments $3$ and $4$:} 
  Dimensionality $d = 10$. \textbf{Experiments $5$ and $6$:} True dimensionality $d^{i_\star} = 5$ and maximal dimensionality $d_M = 15$. In \textbf{Experiments} $3$ and $5$ the action set is the unit sphere. In \textbf{Experiments} $4$ and $6$ the contexts $x_t$ are $10$ actions sampled uniformly from the unit sphere. }
  \label{fig:expected_regret_sample_runs}
  \vspace{-2mm}
\end{figure*}

\section{EXPERIMENTS}\label{sec:experiments}
\vspace{-0.1in}
We evaluate our algorithms on several synthetic benchmarks (environments, base-learners and model selection tasks), and compare their performance against existing meta-learners. For all details of the experimental setup and additional results, see \pref{app:experimental_details}.

These experiments are mostly intended to validate the theory and as a companion to our theoretical results. 
In these experiments, we vary the parameters that we expect to be most important for model selection %
Varying the difficulty of the learning environment itself is something that should mostly be absorbed by base-learners, for example, by choosing base learners operating on more powerful function classes than we do here. Yet, it is important to observe that the meta-algorithms are fairly oblivious to the difficulty of the environment. All that matters here is the regret profile of the base learners. In our experiments, we therefore decided to explore the landscape of model selection by varying the nature of the model selection task itself (dimension, self model selection, and confidence scaling) while keeping the underlying environments fairly simple. 

\textbf{Environments and base-learners:} As the first environment, we use a simple 5-armed multi-armed bandit problem (\textbf{MAB}) with standard Gaussian noise. We then use two linear bandit settings, as also described in \pref{sec:running_example}: linear bandits with stochastic rewards, either with a stochastic context (\textbf{CLB}) or without (\textbf{LB}). As base learners, we
 use \textbf{UCB} for the MAB environment (see also \pref{sec:running_example}) and Linear Thompson (\textbf{LinTS}) sampling \citep{abeille2017linear} for the LB and CLB setting.

\textbf{Model selection task:} We consider 3 different model selection tasks. All the results are reported in~\pref{fig:expected_regret_sample_runs}. In the first, \textbf{conf} (``confidence"), we vary the explore-exploit trade-off in the base learners. For UCB, different base learners correspond to different settings of $c$, the confidence scaling that multiplies the exploration bonus. Analogously, for \textbf{LinTS}, we vary the scale $c$ of the parameter perturbation. 
For the second task \textbf{dim} (``dimension"), we vary the number of dimensions $d_i$ the base learner considers when choosing the action (see second and third example in \pref{sec:running_example}, as well as \pref{fig:expected_regret_sample_runs} for results). Finally, we also consider a ``\textbf{self}'' task, where all base learners are copies of the same algorithm. %

\begin{table*}[t]
\caption{Comparison of meta learners: cumulative regret (averaged over 100 repetitions $\pm 2 \times$standard error) at the end of the sequence of rounds. In bold is the best performer for each environment.
}
\label{tab:general_overview}
\hspace{-0.25cm}
\scalebox{0.85}{
\bgroup
\def\arraystretch{1.1}
\begin{tabular}{|lccc| c| c | c | c | c | c | c |}
\hline
&Env. &Learners & Task & D$^3$RB &  ED$^2$RB & Corral & RB Grid &  UCB &Greedy & EXP3 \\
\hline
 1.& MAB &  UCB &  self &   {$\bf 431 \pm 182 $}  &  {$560 \pm 240$} &  {$5498 \pm 340$}  &  {$6452 \pm 230$} & {$574 \pm 34$} &  {$6404 \pm 1102 $} & {$5892 \pm 356$} \\
 2.& MAB &  UCB &  conf &   $1608 \pm 198$ &  $1413 \pm 208$ &  $2807 \pm 138$  &  $3452 \pm 110$ & $\bf 918 \pm 98$ &  $2505 \pm 362$ &  $3007 \pm 136$\\
 3.& LB &  LinTS &  conf  &   $1150 \pm 134 $ &  $\bf 1135 \pm 148$ &  $2605 \pm 38$  &  $3169 \pm 66$ & $3052 \pm 36$ &  $2553 \pm 302$ &  $2491 \pm 36$\\
 4.& CLB &  LinTS &  conf &  $411 \pm 100$ &  $\bf 406 \pm 94$ &  $1632 \pm 30$  &  $1073 \pm 184$ & $1644 \pm 160$&  $991 \pm 298$&  $1086 \pm 70$ \\
 5.& LB &  LinTS &  dim & $1733 \pm 230$ & $\bf 1556 \pm 198$ &  $3166 \pm 26$  &  $4223 \pm 40$ & $3932 \pm 16$ &  $3385 \pm 306$&  $3315 \pm 20$\\
 6.& CLB &  LinTS &  dim & $\bf2347 \pm 102$ & $ 2365 \pm 96$ &  $5294 \pm 44$  &  $6258 \pm 38$ & $5718 \pm 50$ &  $4778 \pm 506$&  $5742 \pm 46$\\
 \hline
\end{tabular}
\egroup
}
\vspace{-2mm}
\end{table*}

\textbf{Meta-learners:} We evaluate both our algorithms, \textbf{D$^3$RB} and \textbf{ED$^2$RB}, from \pref{alg:balancing}. We compare them against the \textbf{Corral} algorithm~\citep{agarwal2017corralling} with the stochastic wrapper from \citet{pacchiano2020model}, as a representative for FTRL-based meta-learners. We also evaluate  Regret Balanncing from \cite{pacchiano2020model, pmlr-v139-cutkosky21a} with several copies of each base learner, each with a different candidate regret bound, selected on an exponential grid (\textbf{RB Grid}). We also include in our list of competitors three popular algorithms, the \textbf{Greedy} algorithm (always selecting the best base learner so far with no exploration), \textbf{UCB} \citep{10.1023/A:1013689704352} and \textbf{EXP3} \citep{doi:10.1137/S0097539701398375}. These are legitimate choices as meta-algorithms, but either they do not come with theoretical guarantees in the model selection setting (UCB, Greedy) or enjoy worse guarantees \citep{pacchiano2020model}.

\textbf{Discussion.} An overview of our results can be found in \pref{tab:general_overview}, where we report the cumulative regret of each algorithm at the end of each experiment.  \pref{fig:expected_regret_sample_runs} contains the entire learning curves (as regret scale = cumulative regret normalized by $\sqrt{T}$).
We observe that D$^3$RB and ED$^2$RB both outperform all other meta-learners on all but the second benchmark. UCB as a meta-learner performs surprisingly well in benchmarks on MABs but performs poorly on the others. 
Thus, our methods feature the smallest or close to the smallest cumulative regret among meta-learners on all benchmarks. 

Comparing D$^3$RB and ED$^2$RB, we observe overall very similar performance, suggesting that ED$^2$RB may be preferable due to its sharper theoretical guarantee. While the model selection tasks conf and dim are standard in the literature, we also included one experiment with the self task where we simply select among different instances of the same base learner. This task was motivated by our initial observation (see also \pref{fig:expected_regret_motivation}) that base learners have often a very high variability between runs and that model selection can capitalize on. Indeed, \pref{fig:expected_regret_sample_runs} shows that 
our algorithms as well as UCB achieve much smaller overall regret than the base learner. This suggests that model selection can be an effective way to turn a notoriously unreliable algorithm like the base greedy base learner (UCB with $c=0$ is Greedy) into a robust learner.

\vspace{-0.05in}
\section{CONCLUSIONS%
}
\vspace{-0.1in}
We proposed two new algorithms for model selection based on the regret balancing principle but without the need to specify candidate regret bounds a-priori. This calls for more sophisticated regret balancing mechanics that makes our methods data-driven and as an important benefit allows them to capitalize on variability in a base learner's performance. We demonstrate this empirically, showing that our methods perform well across several synthetic benchmarks, as well as theoretically. We prove that both our algorithms achieve regret that is not much worse than the realized regret of any base learner. This data-dependent guarantee recovers existing data-independent results but can be significantly tighter.

In this work, we focused on the fully stochastic setting, with contexts and rewards drawn i.i.d. We believe an extension of our results to arbitrary contexts is fairly easy by replacing the deterministic balancing with a randomized version. In contrast, covering the fully adversarial setting is likely possible by building on top of \citep{pacchiano2022best} but requires substantial innovation.

\bibliography{manual}

\onecolumn
 \appendix
 
 \section*{Appendix}
 The appendix contains the extra material that was omitted from the main body of the paper.  
 
\section{DETAILS ON FIGURE \ref{fig:expected_regret_motivation}}
 \label{app:earlyfigdetails}
We consider a $5$-armed bandit problem with rewards drawn from a Gaussian distribution with standard deviation $6$ and mean $\frac{10}{10}, \frac{6}{10}, \frac{5}{10}, \frac{2}{10}, \frac{1}{10}$ for each arm respectively.
 We use a simple UCB strategy as a base learner that chooses the next action as $\argmax_{a \in \Acal} \hat \mu(a) + c \sqrt{\frac{\ln(n(a) / \delta)}{n(a)}}$ where $n(a)$ and $\hat \mu(a)$ are the number of pulls of arm $a$ so far and the average reward observed.
The base learners use $\delta = \frac{1}{10}$ and $c = 3$ or $c = 4$ respectively.

\section{ANALYSIS COMMON TO BOTH ALGORITHMS}\label{app:common}

\begin{definition}\label{def:evente}
We define the event $\Ecal$ in which we analyze both algorithms as the event in which for all rounds $t \in \NN$ and base learners $i \in [M]$ the following inequalities hold
\begin{align*}
    - c \sqrt{n^i_t \ln \frac{ M \ln n^i_t}{\delta}} \leq \widehat u^i_t - u^i_t \leq c \sqrt{n^i_t \ln \frac{ M \ln n^i_t}{\delta}}
\end{align*}
for the algorithm parameter $\delta \in (0, 1)$ and a universal constant $c > 0$.
\end{definition}
\begin{lemma}\label{lem:highprob}
Event $\Ecal$ from \pref{def:evente} has probability at least $1 - \delta$.
\end{lemma}
\begin{proof}

Consider a fixed $i \in [M]$ and $t$ and write 
\begin{align*}
    \widehat u^i_t - u^i_t 
    &= \sum_{\ell = 1}^t \indicator{i_\ell = i} \left(r_t - v^{\pi_t}\right)
    &= \sum_{\ell = 1}^t \indicator{i_\ell = i} \left(r_\ell - \mathbb E[r_\ell | \pi_\ell] \right)
\end{align*}
Let $\Fcal_t$ be the sigma-field induced by all variables up to round $t$ before the reward is revealed, i.e., $\Fcal_t = \sigma \left( \{x_\ell, \pi_\ell, i_\ell\}_{\ell \in [t-1]} \cup \{x_t, \pi_t, t_t\}\right)$.
Then, $X_\ell = \indicator{i_\ell = i} \left(r_t - \mathbb E[r_t | \pi_t] \right) \in [-1, +1]$ is a martingale-difference sequence w.r.t. $\Fcal_\ell$. We will now apply a Hoeffding-style uniform concentration bound from \citet{howard2018uniform}.
Using the terminology and definition in this article, by case Hoeffding~I in Table~4, the process  $S_t = \sum_{\ell=1}^t X_\ell$ is sub-$\psi_N$ with variance process $V_t = \sum_{\ell=1}^t \indicator{i_\ell = i}/ 4$.
Thus by using the boundary choice in Equation~(11) of \citet{howard2018uniform}, we get
\begin{align*}
    S_t & \leq 1.7 \sqrt{V_t \left( \ln \ln(2 V_t) + 0.72 \ln(5.2 / \delta) \right) }  = 0.85\sqrt{n^i_t\left( \ln \ln(n^i_t/2) + 0.72 \ln(5.2 / \delta)\right)}
\end{align*}
for all $k$ where $V_k \geq 1$ with probability at least $1 - \delta$.
Applying the same argument to $-S_k$ gives that
\begin{align*}
    \left|  \widehat u^i_t - u^i_t  \right|
    \leq 3 \vee 0.85\sqrt{n^i_t\left( \ln \ln(n^i_t /2) + 0.72 \ln(10.4 / \delta)\right)}
\end{align*}
holds with probability at least $1 - \delta$ for all $t$.

We now take a union bound over $i \in [M]$ and rebind $\delta \rightarrow \delta / M$. Then picking the absolute constant $c$ sufficiently large gives the desired statement.
\end{proof}

\begin{lemma}[Balancing potential lemma]\label{lem:balancing_abstract_lemma}
For each $i \in [M]$, let $F_i: \mathbb{N}\cup\{0\} \rightarrow \mathbb{R}_+$ be a nondecreasing potential function that does not increase too quickly, i.e., 
\begin{align*}
F_i(\ell) \leq F_i(\ell+1) \leq \alpha \cdot F_i(\ell) \qquad \forall \ell \in \NN \cup \{0\} %
\end{align*}
and that $0<F_i(0) \leq \alpha \cdot F_j(0)$ for all $i, j \in [M]^2$.
Consider a sequence $(i_t)_{t \in \NN}$ such that  $i_t = \argmin_{i \in [M]} F_i(n^i_{t-1})$ and $n^i_t = \sum_{\ell = 1}^{t} \mathbf{1}\{i_\ell = i\}$, i.e., $i_t \in [M]$ is always chosen as the smallest current potential. Then, for all $t \in \NN$
\begin{equation*}
    \max_{i \in [M]} F_i(n^i_t)  \leq \alpha \cdot \min_{j \in [M]} F_j(n^j_t). 
\end{equation*}
\end{lemma}
\begin{proof}
Our proof works by induction over $t$. At $t = 1$, we have $n^i_0 = 0$ for all $i \in [M]$ and thus, by assumption, the statement holds. Assume now the statement holds for $t$. 
Notice that since $n^i_t$ and $F_i$ are non-decreasing, we have for all $i \in [M]$
\begin{equation*}
     \min_i F_i(n^i_t) \geq \min_i F_i(n^i_{t-1})).
\end{equation*}
Further, for all $i \neq i_t$ that were not chosen in round $t$, we even have $F_i(n^i_{t-1}) = F_i(n^i_t)$ for all $i \neq i_t$. 
We now distinguish two cases:
\paragraph{Case $i_t \notin \argmax_i F_i(n^i_{t-1})$.} Since the potential of all $i \neq i_t$ that attain the $\max$ is unchanged, we have
\begin{align*}
    \max_i F_i(n^i_t) &= \max_i F_i(n^i_{t-1}) 
\end{align*}
and therefore $\frac{\max_i F_i(n^i_t)}{\min_j F_j(n^j_t)} \leq \frac{\max_i F_i(n^i_{t-1})}{\min_j F_j(n^j_{t-1})} \leq \alpha$.

\paragraph{Case $i_t \in  \argmax_i F_i(n^i_{t-1})$.}
Since $i_t$ attains both the maximum and the minimum, and hence all potentials are identical, we have
\begin{equation*}
     \max_i F_i(n^i_t) = F_{i_t}(n^{i_t}_t) \leq  F_{i_t}(n^{i_t}_{t-1} + 1) \leq  \alpha F_{i_t}(n^{i_t}_{t-1}) = \alpha \min_j F_{j}(n^{j}_{t-1}).
\end{equation*}
\end{proof}

\section{PROOFS FOR THE DOUBLING ALGORITHM (ALGORITHM \ref{alg:balancing}, LEFT)}
\label{app:doubling_proofs}
\begin{lemma}\label{lem:dbound}
In event $\Ecal$, for each base learner $i$ all rounds $t \in \NN$, the regret multiplier $\widehat d^i_t$ satisfies 
\begin{align*}
    \widehat d^i_t \leq 2 \bar d^i_t~.
\end{align*}
\end{lemma}
\begin{proof}
Note that instead of showing this for all rounds $t$, we can also show this equivalently for all number $k$ of plays of base learner $i$.
If the statement is violated for base learner $i$, then there is a minimum number $k$ of plays at which this statement is violated.
Note that by definition $\bar d_{(0)}^i = d_{\min}$ and by initialization $\widehat d^i_{(0)} = d_{\min}$, hence this $k$ cannot be $0$.

Consider now the round $t$ where the learner $i$ was played the $k$-th time, i.e., the first round at which the statement was violated.
This means $\widehat d^i_t > 2 \bar d^i_t$ but $\widehat d^i_{t-1} \leq \bar 2 d^i_{t-1}$ still holds. Since $\widehat d^i_t$ can be at most $2\widehat d^i_{t-1}$, we have
$\widehat d^i_{t-1} > \bar d^i_{t}$. We will now show that in this case, the misspecification test could not have triggered and therefore $\widehat d^i_{t} = \widehat d^i_{t-1} \leq 2 \bar d^i_{t-1} \leq 2 \bar d^i_{t}$ which is a contradiction. To show that the test cannot trigger, consider the LHS of the test condition and bound it from below as
\begin{align*}
\frac{\widehat u^{i_t}_t}{n_t^{i_t}} + \frac{\widehat d^{i_{t}}_{t-1} \sqrt{n^{i_t}_t}}{n_t^{i_t}} + c\sqrt{\frac{\ln \frac{M\ln n^{i_t}_t}{\delta}}{n_t^{i_t}}}
    &\geq 
\frac{u^{i_t}_t}{n_t^{i_t}} + \frac{\widehat d^{i_{t}}_{t-1} \sqrt{n^{i_t}_t}}{n_t^{i_t}}  \tag{Event $\Ecal$}
    \\
        &\geq 
        \frac{u^{i_t}_t}{n_t^{i_t}} + \frac{\bar  d^{i_{t}}_{t} \sqrt{n^{i_t}_t}}{n_t^{i_t}} \tag{$\widehat d^{i_{t}}_{t-1} > \bar d^{i_{t}}_{t}$}
        \\
        &\geq     \frac{u^{i_t}_t + \sum_{\ell=1}^{n_t^{i_t}} \reg(\pi^{i_t}_{(\ell)})}{n_t^{i_t}} \tag{definition of $d^{i}_{t}$}
        \\
        & \geq v^\star \tag{definition of regret}
        \\
        & \geq  \frac{u^{j}_t}{n_t^{j}} \tag{definition of $v^\star$}
        \\
        & \geq \frac{u^{j}_t}{n_t^{j}} - c \sqrt{\frac{\ln \frac{M\ln n^{j}_t}{\delta}}{n_t^{j}}}.
        \tag{Event $\Ecal$}
\end{align*}
This holds for any $j \in [M]$ and thus, the test does not trigger.
\end{proof}

\begin{corollary}\label{corr:num_doublings}
In event $\Ecal$, for each base learner $i$ all rounds $t \in \NN$, the number of times the regret multiplier $\widehat d^i_t$ has doubled so far is bounded as follows:
\begin{align*}
\widehat d^i_t \leq  1 + \log_2 \frac{ \bar d^i_{t}}{d_{\min}}~.
\end{align*}
\end{corollary}

\begin{lemma}\label{lem:doubling_balanced}
The potentials in \pref{alg:balancing} (left) are balanced at all times up to a factor $3$, that is,
$\phi^i_{t} \leq 3 \phi^j_{t}$ for all rounds $t \in \NN$ and base learners $i, j \in [M]$.
\end{lemma}
\begin{proof}
We will show that \pref{lem:balancing_abstract_lemma} with $\alpha = 3$ holds when we apply the lemma to $F_i(n^i_{t-1}) = \phi^i_t$. 

First $F_i(0) = \phi^i_1 = d_{\min}$ for all $i \in [M]$ and, thus, the initial condition holds. To show the remaining condition, it suffices to show that $\phi^i_t$ is non-decreasing in $t$ and cannot increase more than a factor of $3$ per round.
If $i$ was not played in round $t$, then $\phi^i_t = \phi^i_{t-1}$ and both conditions holds.
If $i$ was played, i.e., $i = i_t$, then 
\begin{align*}
    \phi^i_t = \widehat d^i_t \sqrt{n^i_t} \leq 2\widehat d^i_{t-1} \sqrt{n^i_t} \leq \begin{cases}
    2\widehat d^i_{t-1} \sqrt{n^i_t - 1} \sqrt{\frac{n^i_t}{n^i_t - 1}} = 2 \phi^i_{t-1} \sqrt{\frac{n^i_t}{n^i_t - 1}} \leq 3 \phi^i_{t-1} & \textrm{if }n^i_t > 1\\
   2\widehat d_{\min} \sqrt{1} = \phi_{t-1}^i & \textrm{if } n^i_t = 1
    \end{cases}
\end{align*}
\end{proof}

\begin{lemma}\label{lem:base_learner_regret}
In event $\Ecal$, the regret of all base learners $i$ is bounded in all rounds $T$ as
\begin{align*}
    \sum_{k = 1}^{n^i_T} \reg(\pi^i_{(k)}) \leq \frac{6(\bar d^j_T)^2}{d_{\min}} \sqrt{n^i_T} 
                    + 6\bar d_T^j \sqrt{n_T^j} 
                    +  \left(6 c\frac{\bar d^j_T}{d_{\min}}  + 2c\right)\sqrt{n_T^i \ln \frac{M\ln T}{\delta}}+ 1 + \log_2 \frac{ \bar d^i_{T}}{d_{\min}}~,
\end{align*}
where $j \in [M]$ is an arbitrary base learner with $n^j_T > 0$.
\end{lemma}
\begin{proof}
Consider a fixed base learner $i$ and time horizon $T$, and let $t \leq T$ be the last round where $i$ was played but the misspecification test did not trigger. If no such round exists, then set $t = 0$. By \pref{corr:num_doublings}, $i$ can be played at most $1 + \log_2 \frac{\bar d^i_{T}}{d_{\min}}$ times between $t$ and $T$ and thus
\begin{align*}
   \sum_{k = 1}^{n^i_T} \reg(\pi^i_{(k)}) \leq \sum_{k = 1}^{n^i_t} \reg(\pi^i_{(k)}) + 1 + \log_2 \frac{\bar  d^i_{T}}{d_{\min}}.
\end{align*}
If $t = 0$, then the desired statement holds. Thus, it remains to bound the first term in the RHS above when $t > 0$. Since $i = i_t$ and the test did not trigger we have, for any base learner $j$ with $n^j_t > 0$,
\begin{align*}
    \sum_{k = 1}^{n^i_t} \reg(\pi^i_{(k)}) &= n^i_t v^\star - u^i_t \tag{definition of regret}\\
    &= n^i_t v^\star - \frac{n^i_t}{n^j_t}u^j_t + \frac{n^i_t}{n^j_t}u^j_t - u^i_t\\
    &=  \frac{n^i_t}{n^j_t}\left( n^j_t v^\star - u^j_t\right) + \frac{n^i_t}{n^j_t}u^j_t - u^i_t\\
    &=  \frac{n^i_t}{n^j_t}\left( \sum_{k = 1}^{n^j_t} \reg(\pi^j_{(k)})\right) + \frac{n^i_t}{n^j_t}u^j_t - u^i_t \tag{definition of regret}\\
    &\leq   \frac{n^i_t}{n^j_t}\left(d^j_t \sqrt{n^j_t}\right) + \frac{n^i_t}{n^j_t}u^j_t - u^i_t \tag{definition of regret rate}\\
    &\leq   \sqrt{\frac{n^i_t}{n^j_t}} d^j_t \sqrt{n^i_t} + \frac{n^i_t}{n^j_t}u^j_t - u^i_t.
\end{align*}
We now use the balancing condition in \pref{lem:doubling_balanced} to bound the first factor $\sqrt{n^i_t / n^j_t}$. This condition gives that $\phi^i_{t+1} \leq 3\phi^j_{t+1}$. Since both $n^j_t > 0$ and $n^i_t > 0$, we have $\phi^i_{t+1} = \widehat d^i_t \sqrt{n^i_t}$ and $\phi^j_{t+1} = \widehat d^j_t \sqrt{n^j_t}$.
Thus, we get
\begin{align}\label{eqn:dn_connection}
    \sqrt{\frac{n^i_t}{n^j_t}} = \sqrt{\frac{n^i_t}{n^j_t}} \cdot \frac{\widehat d^i_t}{\widehat d^j_t} \cdot \frac{\widehat d^j_t}{\widehat d^i_t} = \frac{\phi^i_{t+1}}{\phi^j_{t+1} } \cdot  \frac{\widehat d^j_t}{\widehat d^i_t} \leq 3  \frac{\widehat d^j_t}{\widehat d^i_t} 
    \leq 6 \frac{\bar d^j_t}{d_{\min}}.
\end{align}
Plugging this back into the expression above, we have
\begin{align*}
     \sum_{k = 1}^{n^i_t} \reg(\pi^i_{(k)}) \leq \frac{6(\bar d^j_t)^2}{d_{\min}} \sqrt{n^i_t} + \frac{n^i_t}{n^j_t}u^j_t - u^i_t.
\end{align*}
To bound the last two terms, we use the fact that the misspecification test did not trigger in round $t$. Therefore
\begin{align*}
    u^i_t &\geq \widehat u^i_t - c\sqrt{n_t^i \ln \frac{M\ln n^{i}_t}{\delta}} \tag{event $\Ecal$}\\
    &=n^i_t \left(  \frac{\widehat u^i_t}{n^i_t} + c\sqrt{\frac{\ln \frac{M\ln n^{i}_t}{\delta}}{n^i_t}} + \frac{\widehat d^i_t}{\sqrt{n^i_t}}\right) - 2c\sqrt{n_t^i \ln \frac{M\ln n^{i}_t}{\delta}} - \widehat d_t^i \sqrt{n_t^i}\\
    & \geq \frac{n^i_t}{n^j_t}\widehat u^j_t  - \sqrt{\frac{n^i_t}{n^j_t}} c\sqrt{n^i_t \ln \frac{M\ln n^{j}_t}{\delta}} - 2c\sqrt{n_t^i \ln \frac{M\ln n^{i}_t}{\delta}} - \widehat d_t^i \sqrt{n_t^i} \tag{test not triggered}
\end{align*}
Rearranging terms and plugging this expression in the bound above gives
\begin{align*}
     \sum_{k = 1}^{n^i_t} \reg(\pi^i_{(k)}) &\leq \frac{6( \bar d^j_t)^2}{d_{\min}} \sqrt{n^i_t} +  \sqrt{\frac{n^i_t}{n^j_t}} c\sqrt{n^i_t \ln \frac{M\ln n^{j}_t}{\delta}} + 2c\sqrt{n_t^i \ln \frac{M\ln n^{i}_t}{\delta}} + \widehat d_t^i \sqrt{n_t^i} \\
     &\leq \frac{6(\bar d^j_t)^2}{d_{\min}} \sqrt{n^i_t} +  6 \frac{\bar d^j_t}{d_{\min}} c\sqrt{n^i_t \ln \frac{M\ln n^{j}_t}{\delta}} + 2c\sqrt{n_t^i \ln \frac{M\ln n^{i}_t}{\delta}} + \widehat d_t^i \sqrt{n_t^i} \tag{\pref{eqn:dn_connection}}\\
          &\leq \frac{6(\bar d^j_t)^2}{d_{\min}} \sqrt{n^i_t} +  6 \frac{\bar d^j_t}{d_{\min}} c\sqrt{n^i_t \ln \frac{M\ln n^{j}_t}{\delta}} + 2c\sqrt{n_t^i \ln \frac{M\ln n^{i}_t}{\delta}} + 3\widehat d_t^j \sqrt{n_t^j} \tag{\pref{eqn:dn_connection}}\\
    &\leq \frac{6(\bar d^j_t)^2}{d_{\min}} \sqrt{n^i_t} 
                    + 3\widehat d_t^j \sqrt{n_t^j} 
                    +  \left(6 c\frac{\bar d^j_t}{d_{\min}}  + 2c\right)\sqrt{n_t^i \ln \frac{M\ln t}{\delta}} \tag{$n^i_t \leq t$}\\
    &\leq \frac{6(\bar d^j_t)^2}{d_{\min}} \sqrt{n^i_t} 
                    + 6 \bar d_t^j \sqrt{n_t^j} 
                    +  \left(6 c\frac{\bar d^j_t}{d_{\min}}  + 2c\right)\sqrt{n_t^i \ln \frac{M\ln t}{\delta}} \tag{\pref{lem:dbound}}
\end{align*}
Finally, since $t \leq T$ and therefore $\bar d^j_t \leq \bar d^j_T$ and $n^j_t \leq n^j_T$ (and similarly for $i$), the statement follows.
\end{proof}

\maindouble*
\begin{proof}
By \pref{lem:highprob}, event $\Ecal$ from \pref{def:evente} has probability at least $1 - \delta$. In event $\Ecal$, we can apply \pref{lem:base_learner_regret} for each base learner. 
Summing up the bound from that lemma gives
\begin{align*}
    \Reg(T) &\leq \sum_{i = 1}^M \left[ \frac{6(\bar d^j_T)^2}{d_{\min}} \sqrt{n^i_T} 
                    + 6 \bar d_T^j \sqrt{n_T^j} 
                    +  \left(6 c\frac{\bar d^j_T}{d_{\min}}  + 2c\right)\sqrt{n_T^i \ln \frac{M\ln T}{\delta}}+ 1 + \log_2 \frac{\bar d^i_{T}}{d_{\min}} \right]\\
&\leq 6M \bar d^j_T \sqrt{T} + M + M \log_2 \frac{\sqrt{T}}{d_{\min}} + \left[ \frac{6(\bar d^j_T)^2}{d_{\min}} 
                    +  \frac{4 \bar d^j_T}{d_{\min}} 2c\sqrt{\ln \frac{M\ln T}{\delta}}\right]\sum_{i = 1}^M\sqrt{n_T^i}\\
                    &\leq \left( 6\sqrt{M} \bar d^j_T + \frac{6(\bar d^j_T)^2}{d_{\min}}
                    + \frac{8c \bar d^j_T}{d_{\min}}\sqrt{\ln \frac{M\ln T}{\delta}}
                    \right)\sqrt{MT} + M + M \log_2 \frac{T}{d_{\min}}.
\end{align*}
Plugging in $d_{\min} \geq 1$ yields
\begin{align*}
    \Reg(T) &\leq \left( 6\sqrt{M} \bar d^j_T + 6(\bar d^j_T)^2 + 8c \bar d^j_T\sqrt{\ln \frac{M\ln T}{\delta}}\right)\sqrt{MT} + M + M \log_2 T\\
    & = O \left( \left(M \bar d^j_T + \sqrt{M}(\bar d^j_T)^2 + \bar d^j_T \sqrt{\ln \frac{M\ln T}{\delta}}\right) \sqrt{T}+ M \ln(T)\right)\\
    & = \tilde O \left(\bar d^j_T M\sqrt{T} + (\bar d^j_T)^2 \sqrt{MT}\right)~,
\end{align*}
as desired.
\end{proof}

\section{PROOFS FOR THE ESTIMATING ALGORITHM (ALGORITHM \ref{alg:balancing}, RIGHT)}
\label{app:estimating_proofs}

\begin{lemma}\label{lem:dboundest}
In event $\Ecal$, the regret rate estimate in \pref{alg:balancing} (right) does not overestimate the current regret rate, that is, for all base learners $i \in [M]$ and rounds $t \in \NN$, we have
\begin{align*}
    \widehat d^i_t \leq d^i_t.
\end{align*}
\end{lemma}
\begin{proof}
Note that the algorithm only updates $\widehat d^i_t$ when learner $i$ is chosen and only then $d^i_t$ changes. Further, the condition holds initially since $\widehat d^i_1 = d_{\min} \leq d^i_t$. Hence, it is sufficient to show that this condition holds whenever $\widehat d^i_t$ is updated.
The algorithm estimates $\widehat d^i_t$ as
\begin{align*}
    \widehat d^{i}_{t} = \max\left\{d_{\min}, \,\,\sqrt{n_t^{i}} \left( \max_{j \in [M]} \frac{\hat u^{j}_t}{n_t^{j}} - c\sqrt{\frac{\ln \frac{M\ln n^{j}_t}{\delta}}{n_t^{j}}}
- \frac{\hat u^{i_t}_t}{n_t^{i}} - c\sqrt{\frac{\ln \frac{M\ln n^{i}_t}{\delta}}{n_t^{i}}}\right) \right\}~.
\end{align*}
If $\widehat d^i_t \leq d_{\min}$, then the result holds since by definition $d^i_t \geq d_{\min}$. In the other case, we have
\begin{align*}
     \widehat d^{i}_{t} &= \sqrt{n_t^{i}} \left( \max_{j \in [M]} \frac{\hat u^{j}_t}{n_t^{j}} - c\sqrt{\frac{\ln \frac{M\ln n^{j}_t}{\delta}}{n_t^{j}}}
- \frac{\hat u^{i}_t}{n_t^{i}} - c\sqrt{\frac{\ln \frac{M\ln n^{i}_t}{\delta}}{n_t^{i}}}\right)\\
& \leq \sqrt{n_t^{i}} \left( \max_{j \in [M]} \frac{u^{j}_t}{n_t^{j}} - \frac{u^{i}_t}{n_t^{i}} \right)\tag{event $\Ecal$}\\
& \leq \sqrt{n_t^{i}} \left( v^\star - \frac{u^{i_t}_t}{n_t^{i}} \right)\tag{definition of optimal value $v^\star$}\\
& = \frac{n_t^i v^\star - u^i_t}{\sqrt{n_t^{i}}} = \frac{\sum_{k=1}^{n_t^i}\reg(\pi^i_{(k)})}{\sqrt{n_t^{i}}} \tag{regret definition}\\
& \leq d^i_t~, \tag{definition of $d^i_t$}
\end{align*}
as claimed.
\end{proof}

\begin{lemma}\label{lem:phibound}
In event $\Ecal$, the balancing potentials $\phi^i_t$ in \pref{alg:balancing} (right) satisfy for all $t \in \NN$ and $i \in [M]$ where $n^i_t \geq 1$
\begin{align*}
    \phi^i_{t+1} \leq d^i_t \sqrt{n^i_t}.
\end{align*}
\end{lemma}
\begin{proof}
If $i \neq i_t$, then $\phi^i_{t+1} = \phi^i_t$, $d^i_t = d^i_{t-1}$ and $n^i_t = n^i_{t-1}$. It is therefore sufficient to only check this condition for $i = i_t$.
By definition of the balancing potential, we have when $i = i_t$
\begin{align*}
    \phi^i_{t+1} &\leq \max\left\{ \phi^i_{t}, \widehat d^i_t \sqrt{n^i_t} \right\} \leq \max\left\{ \phi^i_{t}, d^i_t \sqrt{n^i_t}\right\}~,
\end{align*}
where the last inequality holds because of \pref{lem:dboundest}. If $n^i_t = 1$, then $\phi^i_{t} = d_{\min}$ and $d^i_t \sqrt{n^i_t} \geq d_{\min} \sqrt{1}$ by definition, and the statement holds. Otherwise, we can assume that $ \phi^i_{t} \leq d^i_{t-1} \sqrt{n^i_{t-1}}$ by induction. This gives
\begin{align*}
    \phi^i_{t+1}  \leq \max\left\{d^i_{t-1} \sqrt{n^i_{t-1}}, d^i_t \sqrt{n^i_t}\right\}.
\end{align*}
We notice that $ d^i_t \sqrt{n^i_t} = \max\{d_{\min} \sqrt{n^i_t}, \sum_{k=1}^{n^i_t} \reg(\pi^i_{(k)})\}$. Since each term inside the $\max$ is non-decreasing in $t$, $d^i_t \sqrt{n^i_t}$ is also non-decreasing in $t$, and therefore $\phi^i_{t+1}  \leq  d^i_t \sqrt{n^i_t}$, as anticipated.
\end{proof}

\begin{lemma}\label{lem:est_num_double}
In event $\Ecal$, for all $T \in \NN$ and $i \in [M]$, the number of times the balancing potential $\phi^i_t$ doubled until time $T$ in \pref{alg:balancing} (right) is bounded by
\begin{align*}
    \log_2 \left(t \max\{1, 1 / d_{\min}\} \right).
\end{align*}
\end{lemma}
\begin{proof}
The balancing potential $\phi^i_t$ is non-decreasing in $t$ and $\phi^i_1 = d_{\min}$. Further, by \pref{lem:phibound}, we have
\begin{align*}
    \phi^i_{t+1} \leq d^i_t \sqrt{n^i_t} \leq  \max\left\{d_{\min} \sqrt{n^i_t} , n^i_t\right\}.
\end{align*}
Thus, the number of times $\phi^i_t$ can double is at most
\begin{align*}
    \log_2 \left( \max\left\{\sqrt{n^i_t} , \frac{n^i_t}{d_{\min}}\right\} \right) \leq \log_2 \left(t \max\{1, 1 / d_{\min}\} \right)~.
\end{align*}
\end{proof}

\begin{lemma}\label{lem:estimating_balanced}
The balancing potentials in \pref{alg:balancing} (right) are balanced at all times up to a factor $2$, that is,
$\phi^i_{t} \leq 2 \phi^j_{t}$ for all rounds $t \in \NN$ and base learners $i, j \in [M]$.
\end{lemma}
\begin{proof}
We will show that \pref{lem:balancing_abstract_lemma} with $\alpha = 2$ holds when we apply the lemma to $F_i(n^i_{t-1}) = \phi^i_t$. 

First $F_i(0) = \phi^i_1 = d_{\min}$ for all $i \in [M]$ and, thus, the initial condition holds. To show the remaining condition, it suffices to show that $\phi^i_t$ is non-decreasing in $t$ and cannot increase more than a factor of $2$ per round. This holds by the clipping in the definition of $\phi^i_{t+1}$ in the algorithm.
\end{proof}

\begin{lemma}\label{lem:base_learner_regret_est}
In event $\Ecal$, the regret of all base learners $i$ is bounded in all rounds $T$ as
\begin{align*}
   \sum_{k = 1}^{n^i_T} \reg(\pi^i_{(k)}) 
     \leq 
     \frac{2(d^j_t)^2}{d_{\min}} \sqrt{n^i_t} + 2 d^j_t \sqrt{n^j_t} + 2c\left(1 + \frac{2d^j_t}{d_{\min}}\right)\sqrt{n^i_t \ln \frac{M\ln t}{\delta}} + \log_2 \max\left\{T, \frac{T}{ d_{\min}}\right\} ,
\end{align*}
where $j \in [M]$ is an arbitrary base learner with $n^j_T > 0$ and $t \leq T$ is the last round where $i = i_t$ and $\phi^i_{t+1} < 2\phi^i_t$.
\end{lemma}
\begin{proof}
Consider fixed base learner $i$ and time horizon $T$, and let $t \leq T$ be the last round where $i$ was played and $\phi^i_t$ did not double, i.e., $\phi^i_{t+1} < 2\phi^i_t$. If no such round exists, then set $t = 0$. By \pref{lem:est_num_double}, $i$ can be played at most $\log_2 \left(T \max\{1, 1 / d_{\min}\} \right)$ times between $t$ and $T$ and thus
\begin{align*}
   \sum_{k = 1}^{n^i_T} \reg(\pi^i_{(k)}) \leq \sum_{k = 1}^{n^i_t} \reg(\pi^i_{(k)}) + \log_2 \left(T \max\{1, 1 / d_{\min}\} \right).
\end{align*}
If $t = 0$, then the desired statement holds. Thus, it remains to bound the first term above when $t > 0$. We can write the regret of base learner $i$ up to $t$ in terms of the regret of any learner $j$ with $n^j_t > 0$ as follows:
\begin{align*}
    \sum_{k = 1}^{n^i_t} \reg(\pi^i_{(k)}) &= n^i_t v^\star - u^i_t \tag{definition of regret}\\
    &= n^i_t v^\star - \frac{n^i_t}{n^j_t}u^j_t + \frac{n^i_t}{n^j_t}u^j_t - u^i_t\\
    &=  \frac{n^i_t}{n^j_t}\left( n^j_t v^\star - u^j_t\right) + \frac{n^i_t}{n^j_t}u^j_t - u^i_t\\
    &=  \frac{n^i_t}{n^j_t}\left( \sum_{k = 1}^{n^j_t} \reg(\pi^j_{(k)})\right) + \frac{n^i_t}{n^j_t}u^j_t - u^i_t \tag{definition of regret}\\
    &\leq   \frac{n^i_t}{n^j_t}\left(d^j_t \sqrt{n^j_t}\right) + \frac{n^i_t}{n^j_t}u^j_t - u^i_t \tag{definition of regret rate}\\
    &\leq   \sqrt{\frac{n^i_t}{n^j_t}} d^j_t \sqrt{n^i_t} + \frac{n^i_t}{n^j_t}u^j_t - u^i_t.
\end{align*}
We now use the balancing condition in \pref{lem:estimating_balanced} to bound the first factor $\sqrt{n^i_t / n^j_t}$. This condition gives that $\phi^i_{t+1} \leq 2\phi^j_{t+1}$. 
Since $\phi^i_{t+1} < 2\phi^i_t$ and, thus, the balancing potential was not clipped from above, we have $\phi^i_{t+1} \geq \widehat d^i_t \sqrt{n^i_t}$. Further, 
since $n^j_t > 0$ we can apply \pref{lem:phibound} to get $\phi^j_{t+1} \leq d^j_t \sqrt{n^j_t}$.
Thus, we get
\begin{align}\label{eqn:dn_connection_est}
    \sqrt{\frac{n^i_t}{n^j_t}} = \sqrt{\frac{n^i_t}{n^j_t}} \cdot \frac{\widehat d^i_t}{d^j_t} \cdot \frac{ d^j_t}{\widehat d^i_t} \leq \frac{\phi^i_{t+1}}{\phi^j_{t+1} } \cdot  \frac{d^j_t}{\widehat d^i_t} \leq 2  \frac{d^j_t}{\widehat d^i_t} 
    \leq 2 \frac{d^j_t}{d_{\min}}.
\end{align}
Plugging this back into the expression above, we have
\begin{align*}
     \sum_{k = 1}^{n^i_t} \reg(\pi^i_{(k)}) \leq \frac{2(d^j_t)^2}{d_{\min}} \sqrt{n^i_t} + \frac{n^i_t}{n^j_t}u^j_t - u^i_t.
\end{align*}
To bound the last two terms, we use the regret coefficient estimate:
\begin{align*}
\frac{n^i_t}{n^j_t}u^j_t - u^i_t
&= n^i_t \left(\frac{u^j_t}{n^j_t} - \frac{u^i_t}{n^i_t} \right)\\
& \leq n^i_t \left(\frac{\hat u^j_t}{n^j_t} - \frac{\hat u^i_t}{n^i_t} \right) 
+ c\sqrt{n^i_t \ln \frac{M\ln n^{i}_t}{\delta}} + c n^i_t \sqrt{\frac{\ln \frac{M\ln n^{j}_t}{\delta}}{n_t^{j}}}
\tag{event $\Ecal$}\\
& = n^i_t \left(\frac{\hat u^j_t}{n^j_t} - c  \sqrt{\frac{\ln \frac{M\ln n^{j}_t}{\delta}}{n_t^{j}}} - \frac{\hat u^i_t}{n^i_t}  - c\sqrt{\frac{\ln \frac{M\ln n^{i}_t}{\delta}}{n_t^{i}}}\right) 
+ 2c\sqrt{n^i_t \ln \frac{M\ln n^{i}_t}{\delta}} + 2c n^i_t \sqrt{\frac{\ln \frac{M\ln n^{j}_t}{\delta}}{n_t^{j}}}\\
 & \leq \widehat d^i_t \sqrt{n^i_t} 
+ 2c\sqrt{n^i_t \ln \frac{M\ln n^{i}_t}{\delta}} + 2c n^i_t \sqrt{\frac{\ln \frac{M\ln n^{j}_t}{\delta}}{n_t^{j}}} \tag{definition of $\widehat d^i_t$}\\
& \leq \widehat d^i_t \sqrt{n^i_t} + 2c\left(1 + \sqrt{\frac{n^i_t}{n^j_t}} \right)\sqrt{n^i_t \ln \frac{M\ln t}{\delta}} \tag{$n^i_t \leq t$ and $n^j_t \leq t$}\\
& \leq \widehat d^i_t \sqrt{n^i_t} + 2c\left(1 + 2 \frac{d^j_t}{d_{\min}}\right)\sqrt{n^i_t \ln \frac{M\ln t}{\delta}} \tag{\pref{eqn:dn_connection_est}}\\
& \leq \phi^i_{t+1} + 2c\left(1 + 2 \frac{d^j_t}{d_{\min}}\right)\sqrt{n^i_t \ln \frac{M\ln t}{\delta}} \tag{$\phi^i_{t+1} \geq \widehat d^i_t \sqrt{n^i_t}$}\\
& \leq 2\phi^j_{t+1} + 2c\left(1 + 2 \frac{d^j_t}{d_{\min}}\right)\sqrt{n^i_t \ln \frac{M\ln t}{\delta}} \tag{\pref{lem:estimating_balanced}}\\
& \leq 2 d^j_t \sqrt{n^j_t} + 2c\left(1 + 2 \frac{d^j_t}{d_{\min}}\right)\sqrt{n^i_t \ln \frac{M\ln t}{\delta}}. \tag{\pref{lem:phibound}}
\end{align*}
Plugging this back into the expression above, we get the desired statement:
\begin{align*}
     \sum_{k = 1}^{n^i_T} \reg(\pi^i_{(k)}) 
     \leq 
     \frac{2(d^j_t)^2}{d_{\min}} \sqrt{n^i_t} + 2 d^j_t \sqrt{n^j_t} + 2c\left(1 + \frac{2d^j_t}{d_{\min}}\right)\sqrt{n^i_t \ln \frac{M\ln t}{\delta}} + \log_2 \max\left\{T, \frac{T}{ d_{\min}}\right\}~. 
\end{align*}
\end{proof}

\mainestimate*
\begin{proof}
By \pref{lem:highprob}, event $\Ecal$ from \pref{def:evente} has probability at least $1 - \delta$. In event $\Ecal$, we can apply \pref{lem:base_learner_regret_est} for each base learner. 
Summing up the bound for all base learners $i \in [M]$ with $j \in \argmin_{i' \in [M]} \max_{i} d^{i'}_{T_{i}}$ from that lemma gives
\begin{align*}
    \Reg(T) &\leq \sum_{i = 1}^M \left[ \frac{2(d^j_{T_i})^2}{d_{\min}} \sqrt{n^i_{T_i}} + 2 d^j_{T_i} \sqrt{n^j_{T_i}} + 2c\left(1 + \frac{2d^j_{T_i}}{d_{\min}}\right)\sqrt{n^i_{T_i} \ln \frac{M\ln T}{\delta}} + \log_2 \max\left\{T, \frac{T}{ d_{\min}}\right\} \right]\\
&\leq 2M d^\star_T \sqrt{T} + M \log_2 \max\left\{T, \frac{T}{ d_{\min}}\right\} + \left[ \frac{2(d^\star_T)^2}{d_{\min}} 
                    +  \frac{6 d^\star_T}{d_{\min}} c\sqrt{\ln \frac{M\ln T}{\delta}}\right]\sum_{i = 1}^M\sqrt{n_T^i}\\
                    &\leq \left( 2\sqrt{M}  d^\star_T + \frac{2( d^\star_T)^2}{d_{\min}}
                    + \frac{6c  d^\star_T}{d_{\min}}\sqrt{\ln \frac{M\ln T}{\delta}}
                    \right)\sqrt{MT} + M \log_2 \max\left\{T, \frac{T}{ d_{\min}}\right\}.
\end{align*}
Here we have used that $d^j_{T_i} \leq d^\star_T$ for all $i \in [M]$ by the definition of $d^\star_T$.
Plugging in $d_{\min} \geq 1$ gives 
\begin{align*}
    \Reg(T) &\leq \left( 2\sqrt{M} d^\star_T + 2(d^\star_T)^2 + 6c  d^\star_T\sqrt{\ln \frac{M\ln T}{\delta}}\right)\sqrt{MT} +  M \log_2 T\\
    & = O \left( \left(M d^\star_T + \sqrt{M}( d^\star_T)^2 + d^\star_T \sqrt{M \ln \frac{M\ln T}{\delta}}\right) \sqrt{T}+ M \ln(T)\right)\\
    & = \tilde O \left( d^\star_T M\sqrt{T} + (d^\star_T)^2 \sqrt{MT}\right)~,
\end{align*}
as claimed.
\end{proof}

\section{EXPERIMENTAL DETAILS}\label{app:experimental_details}

We used a 50 core machine to run our experiments. We made use of this computing infrastructure by parallelizing our experiment runs. The experiments %
take ~12 hours to complete. 

\subsection{Meta-Learners}
We now list the meta-learners used in our experiments.

\begin{algorithm}[H]
\textbf{Input:} $M$ base learners, learning rate $\eta$.  \\
Initialize: $\gamma = 1/T, \beta = e^{\frac{1}{\ln T}}, \eta_{1,j} = \eta, \rho^j_{1} = 2M, \underline{p}^j_{1} = \frac{1}{ \rho^j_{1}},  {p}^j_1= 1/M$ for all $j \in [M]$.\\ 
\For{$t = 1, \cdots, T$}{
Sample $i_t \sim p_t$. \\
Receive feedback $r_t$ from base learner $i_t$. \\ 
Update $p_{t}$, $\eta_t$, $\underline{p}_t$ and $\rho_t$ to $p_{t+1}$, $\eta_{t+1}$, $\underline{p}_{t+1}$ and $\rho_{t+1}$ using $\mathsf{CORRAL-Update}$ Algorithm~\ref{Alg:corral_update}. \\
}
\caption{CORRAL Meta-Algorithm}
\label{Alg:corral_meta-algorithm}
\end{algorithm}

\begin{algorithm}[h]
\textbf{Input:} learning rate vector $\eta_t$, previous distribution $p_t$ and current loss $\ell_t$ \\
\textbf{Output:} updated distribution $p_{t+1}$ \\
Find $\lambda \in [\min_{j} \ell_{t,j} , \max_{j} \ell_{t,j} ]$ such that $\sum_{j=1}^M \frac{ 1}{ \frac{1}{p^i_{t}} + \eta_{t,j}(\ell_{t,j} - \lambda)   }= 1$\\
Return ${p}_{t+1}$ such that $\frac{1}{p^j_{t+1}} = \frac{1}{p^j_{t}} + \eta_{t,j}(\ell_{t,j} - \lambda)$\\
 \caption{Log-Barrier-OMD($p_t, \ell_t, \eta_t$) }
\label{Alg:log_barrier}
\end{algorithm}

\begin{algorithm}[H]
\textbf{Input:} learning rate vector $\eta_t$, distribution $p_t$, lower bound $\underline{p}_t$ and current loss $r_t$ \\
\textbf{Output:} updated distribution $p_{t+1}$, learning rate $\eta_{t+1}$ and loss range $\rho_{t+1}$ \\
Update $p_{t+1} = \text{Log-Barrier-OMD}(p_t, \frac{r_{t}}{{p}_{t,j_t}}\mathbf{e}_{j_t}, \eta_t)$. \\
Set ${p}_{t+1} = (1-\gamma)p_{t+1} + \gamma \frac{1}{M}$. 

\For{$j =1, \cdots, M$}{
\If{$ \underline{p}^j_{t} > {{p}^j_{t+1}} $}{
Set $\underline{p}^j_{t+1} = \frac{{p}^j_{t+1}}{2}, \eta_{t+1, j} = \beta\eta_{t,i}$, \\
}
\Else{
 Set $\underline{p}^j_{t+1}=\underline{p}^j_{t}, \eta_{t+1,j} = \eta_{t,i}$. \\
}
Set $\rho^j_{t+1} = \frac{1}{\underline{p}^j_{t+1}}$.\\
}
Return $p_{t+1}$, $\eta_{t+1}$, $\underline{p}_{t+1}$ and $\rho^j_{t+1}$. 
\caption{$\mathsf{CORRAL-Update}$}\label{Alg:corral_update}
\end{algorithm}

\paragraph{Corral.} We used the \textbf{Corral} Algorithm as described in~\cite{agarwal2017corralling} and~\cite{pacchiano2020model}. Since we work with stochastic base algorithms we use the Stochastic Corral version of~\cite{pacchiano2020model} where the base algorithms are updated with the observed reward $r_t$ instead of the importance sampling version required by the original \textbf{Corral} algorithm of~\cite{agarwal2017corralling}. The pseudo-code is in Algorithm~\ref{Alg:corral_meta-algorithm}. In accordance with theoretical results we set $\eta = \Theta(\frac{1}{\sqrt{T}} ) $. We test the performance of the \textbf{Corral} meta-algorithm with different settings of the initial learning rate $\eta \in \{  .1/\sqrt{T}, 1/\sqrt{T}, 10/\sqrt{T} \}  $. In the table and plots below we call them \textbf{CorralLow}, \textbf{Corral} and \textbf{CorralHigh} respectively. In \pref{tab:exp3_overview_appendix} we compare their performance on different experiment benchmarks. We see \textbf{Corral} and \textbf{CorralHigh} achieve a better formance than \textbf{CorralLow}. The performance of \textbf{Corral} and \textbf{CorralHigh} is similar.

\paragraph{EXP3.} At the beginning of each time step the \textbf{EXP3} meta-algorithm samples a base learner index $i_t \sim p_t$ from its base learner distribution $p_t$. The meta-algorithm maintains importance weighted estimator of the cumulative rewards for each base learner $R_t^{i}$ for all $i \in [M]$. After receiving feedback $r_t$ from base learner $i_t$ the importance weighted estimators are updated as $R_{t+1}^{i} = R_{t}^{i} + \mathbf{1}(i = i_t) \frac{r_t}{p_t^{i_t}}$. The distribution $p_{t+1}^i = (1-\gamma)\exp( \eta R_{t+1}^i )/\sum_{i'} \exp(\eta R_{t+1}^{i'}) + \gamma/M$ where $\eta$ is a and $\gamma$ are a learning rate and exploration parameters. In accordance with theoretical results (see for example \citep[Th.~11.1]{lattimore2020bandit}) in our experiments we set the learning rate to $\eta = \sqrt{\frac{\log(M)}{MT}}$ and set the forced exploration parameter $\gamma = \frac{0.1}{\sqrt{T}}$. We test the performance of the \textbf{EXP3} meta-algorithm with different settings of the forced exploration parameter $\gamma \in \{0, \frac{.1}{\sqrt{T}}, \frac{1}{\sqrt{T}} \}$. In \pref{tab:exp3_overview_appendix} we call them \textbf{EXP3Low}, \textbf{EXP3} and \textbf{EXP3High}. All these different variants have a similar performance.

\paragraph{Greedy.} This is a pure exploitation meta-learner. After playing each base learner at least once, the \textbf{Greedy} meta-algorithm maintains the same cumulative reward statistics $\{ \widehat u^{i}_t \}_{i \in [M]}$ as D$^3$RB and ED$^2$RB.  The base learner $i_t$ chosen at time $t$ is $i_t = \argmax_{i\in [M]} \frac{u^{i}_t}{n_t^i}$.

\paragraph{UCB.} We use the same \textbf{UCB} algorithm as described in \pref{sec:running_example}. We set the scaling parameter $c = 1$.

\paragraph{D$^3$RB and ED$^2$RB.} These are the algorithms in Algorithm \ref{alg:balancing}. We set therein $c = 1$ and $d_{\min} = 1$.

\subsection{Base Learners}
All base learners have essentially been described, except for the Linear Thompson Sampling Algorithm ($\mathrm{LinTS}$) algorithm, which was used in all our linear experiments.

In our implementation we use the algorithm described as in~\cite{abeille2017linear}. On round $t$ the Linear Thompson Sampling algorithm has played $x_1, \cdots x_{t-1} \subset \mathbb{R}^d$ with observed responses $r_1, \cdots, r_{t-1}$. The rewards are assumed to be of the form $r_\ell = x_\ell^\top \theta_\star + \xi_t$ for an unknown vector $\theta_\star$ and a conditionally zero mean random variable $\xi_t$. An empirical model of the unknown vector $\theta_\star$ is produced by fitting a ridge regression least squares estimator  $\widehat{\theta}_t = \argmin_{\theta} \lambda \| \theta\|^2 +  \sum_{\ell=1}^{t-1} ( x_\ell^\top \theta - r_\ell)^2$ for a user specified parameter $\lambda > 0$. This can be written in closed form as $\widehat{\theta}_t = \left(  \mathbf{X}^\top \mathbf{X} + \lambda \mathbb{I}\right)^{-1} \mathbf{X}^\top y$ where $\mathbf{X} \in \mathbb{R}^{t-1\times d}$ matrix where row $\ell$ equals $x_\ell$. At time $t$ a sample model is computed $\widetilde{\theta}_t= \widehat{\theta}_t + c \sqrt{d} \left(  \mathbf{X}^\top \mathbf{X} + \lambda \mathbb{I}\right)^{-1/2} \boldsymbol{\eta}_t$ where $\boldsymbol{\eta}_t \sim \mathcal{N}(\mathbf{0}, \mathbb{I})$ and $c > 0$ is a confidence scaling parameter. This is one of the parameters that we vary in our experiments. If the action set at time $t$ equals $\mathcal{A}_t$ (in the contextual setting $\mathcal{A}_t$ changes every time-step while in the fixed action set linear bandits case it ) the action $x_t = \argmax_{ x \in \mathcal{A}_t} x_t^\top \widetilde{\theta}_t$. In our experiments $\lambda = 1$ and $\theta_\star$ is set to a scaled version of the vector $(0, \cdots, d-1)$. In the detailed experiment description below we specify the precise value of $\theta_\star$ in each experiment.

\begin{figure}[t!]
\vspace{-.3cm}
\begin{center}
\hspace{-1.0cm}
\scalebox{0.85}{
\begin{tikzpicture}[level distance=1.5cm,
  level 1/.style={sibling distance=9cm},
  level 2/.style={sibling distance=5cm},
  level 3/.style={sibling distance=1.5cm},
  level 4/.style={sibling distance=1cm}]
  \node {Model Selection Experiments}
    child {node(A) {MAB}
      child {node[text width=2cm, align=center](B) {\small Self Model Selection}
            child{node[text width=1cm, align=center](C){\small Gaussian \textcolor{red}{\textbf{1}} }}
            child{node[text width=1cm, align=center](D){\small Bernoulli \textbf{B} }}
      }
      child {node[text width=2cm, align=center](E) {\small Confidence Scalings}
        child{node[text width=1cm, align=center](F){\small Gaussian \textcolor{red}{\textbf{2}}    }}
        child{node[text width=1cm, align=center](G){\small Bernoulli \textbf{A} }}
      }
    }
    child {node { Linear Bandits}
    child {node[text width=2cm, align=center] {\small Nested Dimension}
        child{node[text width=1cm, align=center](H){\small Sphere  \textcolor{red}{\textbf{5}}, \textbf{J}      }}
        child{node[text width=1cm, align=center](I){\small Hypercube \textbf{K, L}  }}
        child{node[text width=1cm, align=center](J){\small Contextual \textcolor{red}{\textbf{6}} }}
    }
      child {node[text width=2cm, align=center] {\small Confidence Scalings}
        child{node[text width=1cm, align=center](K){\small Sphere   \textcolor{red}{\textbf{3}}, \textbf{F, G}     }}
        child{node[text width=1cm, align=center](L){\small Hypercube    \textbf{C, D, E}  }}
        child{node[text width=1cm, align=center](M){\small Contextual \textcolor{red}{\textbf{4}}, \textbf{H, I}     }}
      }
    };
\end{tikzpicture}
}
\caption{Experiment Map.} \label{fig::binary_tree}
\end{center}
\end{figure}
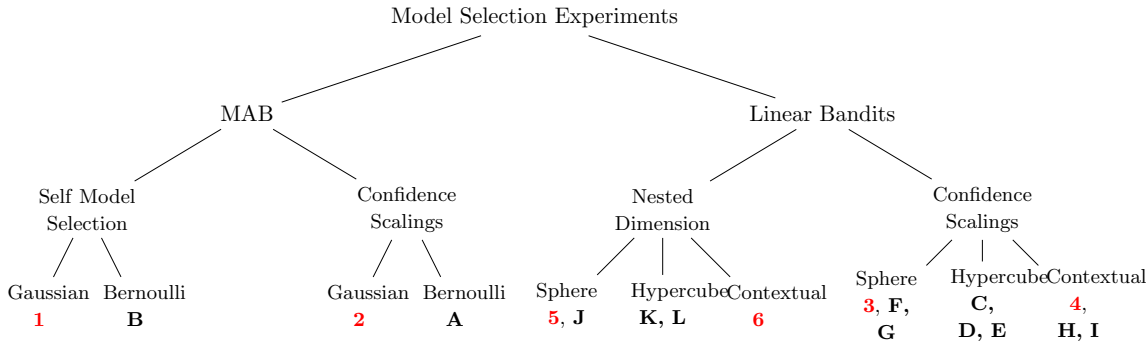

\subsection{Detailed Experiments Description}\label{app:detailexperiments}
\begin{figure}[t!]
  \centering
    \hspace{-1.5cm}\includegraphics[width=0.375\textwidth]{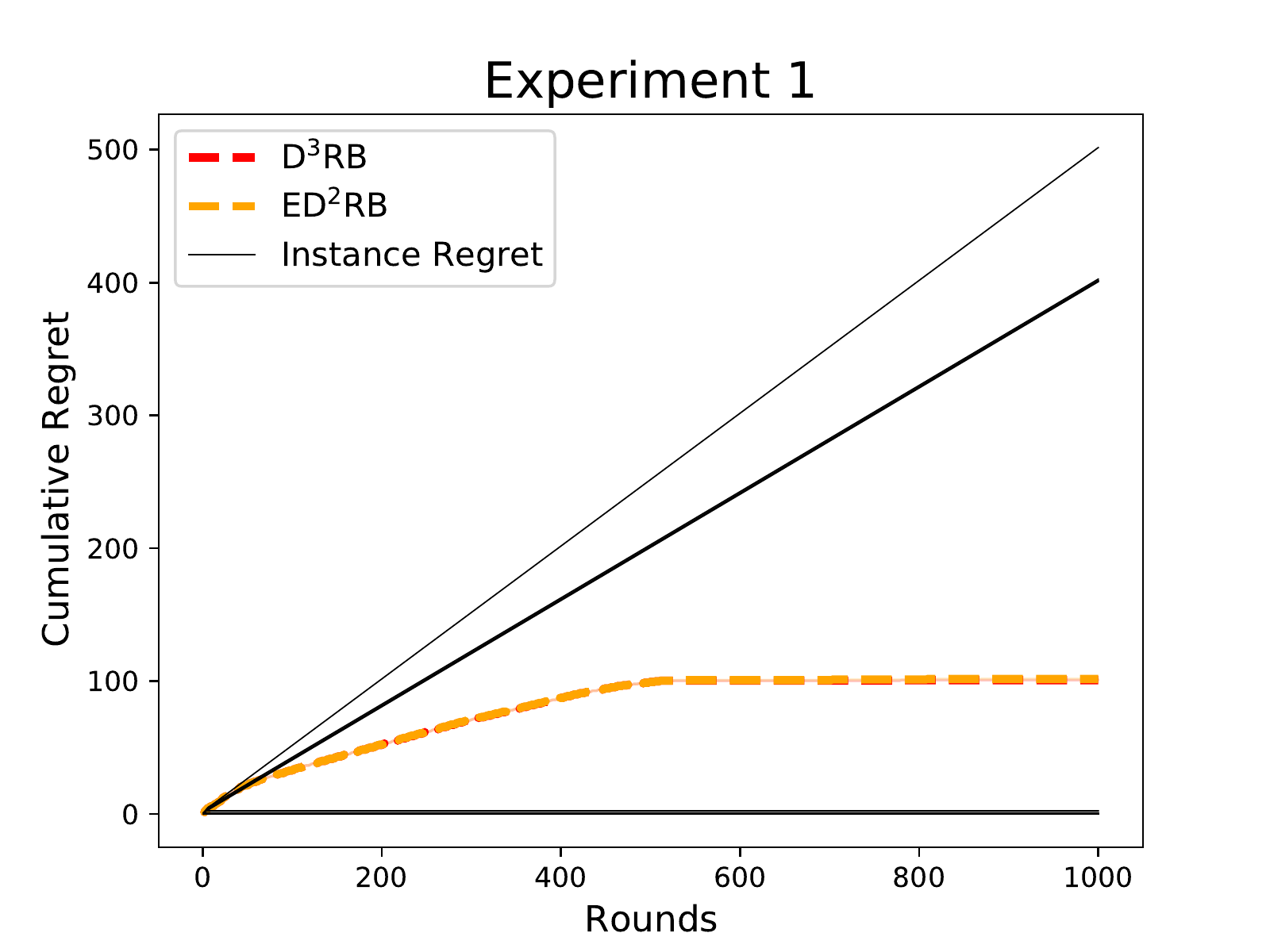}
  \vspace{-0.1in}
  \caption{Experiment $1$ (see \pref{tab:general_overview} for reference). Regret for \textbf{D$^3$RB} and \textbf{ED$^2$RB} (\pref{alg:balancing}) on a single realization.}
  \label{fig:expected_regret_sample_runs_appendix}
\end{figure}

Figure \ref{fig::binary_tree} illustrates the overall structure of our experiments. Experiments $1$ through $6$ are those also reported in the main body of the paper. The below table contains a detailed description of each experiment, together with the associated evidence in the form of learning curves (regret scale vs. rounds). Finally, Table \ref{tab:general_overview_appendix} contains the final (average) cumulative regret for each meta-learner on each experiment.

\begin{table}[htbp]
    \centering
    \begin{tabular}{|p{5cm}|p{10cm}|}
        \hline
        \textbf{Description} & \textbf{Figure} \\
        \hline
        \begin{minipage}[t]{5cm}
        \vspace{-5.7cm}
            \textbf{Experiment 1.} Self model selection for a $5$ armed Gaussian bandit problem with means $[ .5, 1, .2, .1, .6]$ and standard deviations equal to $1$. The base learners are $\mathrm{UCB}$ algorithms with a confidence scaling $c=0$. This reduces them to instances of \textbf{Greedy}. It is well known that greedy algorithms do not satisfy a sublinear expected regret bound. These results show that combining multiple instances of \textbf{Greedy} via self model selection can turn them into algorithms with a sublinear regret guarantee. We initialize $10$ \textbf{Greedy} base learners. 
        \end{minipage} &
        \begin{minipage}[t]{10cm}
            \includegraphics[width=\linewidth]{figs/expectedregret/norm_exp7_DoublingDataDrivenEstimatingDataDrivenCorralBalancingClassicUCBGreedyEXP3_T20000.pdf}
        \end{minipage} \\
        \hline
        \begin{minipage}[t]{5cm}
        \vspace{-5.7cm}
        \paragraph{Experiment 2.} Model selection for a $5$ armed Gaussian bandit problem with means $[ .5, 1, .2, .1, .6]$ and standard deviations equal to $1$ where we select among $4$ UCB base learners with confidence scalings in $\{0,4,6,20\}$. 
        \end{minipage} &
        \begin{minipage}[t]{10cm}
            \includegraphics[width=\linewidth]{figs/confidenceMAB/norm_exp7multiple_DoublingDataDrivenEstimatingDataDrivenCorralBalancingClassicUCBGreedyEXP3_T10000.pdf}
        \end{minipage} \\
        \hline
        \begin{minipage}[t]{5cm}
        \vspace{-5.7cm}
        \paragraph{Experiment 3.} Linear bandits model selection where we select among LinTS base learners with different confidence scalings. The action set is the $10$ dimensional unit sphere. The $\theta_\star$ vector equals $(0,\ldots, 9)/\|(0,\ldots, 9)\| * 5$. The base learners are LinTS instances with confidence scalings in $\{ 0,.16, 2.5, 5, 25\}$.
        \end{minipage} &
        \begin{minipage}[t]{10cm}
            \includegraphics[width=\linewidth]{figs/confidenceLinear/norm_linear1big_DoublingDataDrivenEstimatingDataDrivenCorralBalancingClassicUCBGreedyEXP3_T1000.pdf}
        \end{minipage} \\
        \hline
    \end{tabular}
    \label{tab:figures1}
\end{table}

\begin{table}[htbp]
    \centering
    \begin{tabular}{|p{5cm}|p{10cm}|}
        \hline
        \textbf{Description} & \textbf{Figure} \\
        \hline
        \begin{minipage}[t]{5cm}
        \vspace{-5.7cm}
           \paragraph{Experiment 4.} Contextual linear bandits model selection where we select among LinTS base learners with different confidence scalings. The contexts are generated by producing $10$ i.i.d. uniformly distributed vectors from the unit sphere. The ambient space dimension is $d = 10$.  The $\theta_\star$ vector is $(0,\ldots, 9)$. The base learners are LinTS instances with confidence scalings in $\{ 0,.16, 2.5, 5, 25\}$.
        \end{minipage} &
        \begin{minipage}[t]{10cm}
            \includegraphics[width=\linewidth]{figs/confidenceLinear/norm_linear3_DoublingDataDrivenEstimatingDataDrivenCorralBalancingClassicUCBGreedyEXP3_T1000.pdf}
        \end{minipage} \\
        \hline
        \begin{minipage}[t]{5cm}
        \vspace{-5.7cm}
        \paragraph{Experiment 5.} Nested linear bandits model selection where we select among different LinTS base learners with different ambient dimensions. The action set is the unit sphere and the true ambient dimension equals $5$. The $\theta_\star$ vector is $(0,1,2,3,4)$ and the base learners are LinTS instances with dimensions $d = 2, 5, 10, 15$, and confidence scaling $2$. 
        \end{minipage} &
        \begin{minipage}[t]{10cm}
            \includegraphics[width=\linewidth]{figs/nestedLinear/norm_nestedlinear1_DoublingDataDrivenEstimatingDataDrivenCorralBalancingClassicUCBGreedyEXP3_T1000.pdf}
        \end{minipage} \\
        \hline
        \begin{minipage}[t]{5cm}
        \vspace{-5.7cm}
        \paragraph{Experiment 6.} Nested contextual linear bandits model selection where we select among different LinTS base learners with different ambient dimensions. The context set is generated by sampling $10$ i.i.d. vectors from the unit sphere. The true ambient dimension equals $5$. The $\theta_\star$ vector equals $(0,1,2,3,4)/ \|  (0,1,2,3,4)  \|$ and the base learners are LinTS instances with dimensions $d = 2, 5, 10, 15$  and confidence scaling $2$.
        \end{minipage} &
        \begin{minipage}[t]{10cm}
            \includegraphics[width=\linewidth]{figs/nestedLinear/norm_nestedlinear3_DoublingDataDrivenEstimatingDataDrivenCorralBalancingClassicUCBGreedyEXP3_T20000.pdf}
        \end{minipage} \\
        \hline
    \end{tabular}
    \label{tab:figures2}
\end{table}

\begin{table}[htbp]
    \centering
    \begin{tabular}{|p{5cm}|p{10cm}|}
        \hline
        \textbf{Description} & \textbf{Figure} \\
        \hline
        \begin{minipage}[t]{5cm}
        \vspace{-5.7cm}
\paragraph{Experiment A.} Selecting among different confidence scalings in a $4$ armed Bernoulli bandit problem with mean rewards $[.1, .2, .5, .8]$. We select among $9$ \textbf{UCB} base learners with confidence scalings in $\{ 0, .08, .16, .64, 1.24, 2.5, 5, 10, 25  \}$. 
        \end{minipage} &
        \begin{minipage}[t]{10cm}
            \includegraphics[width=\linewidth]{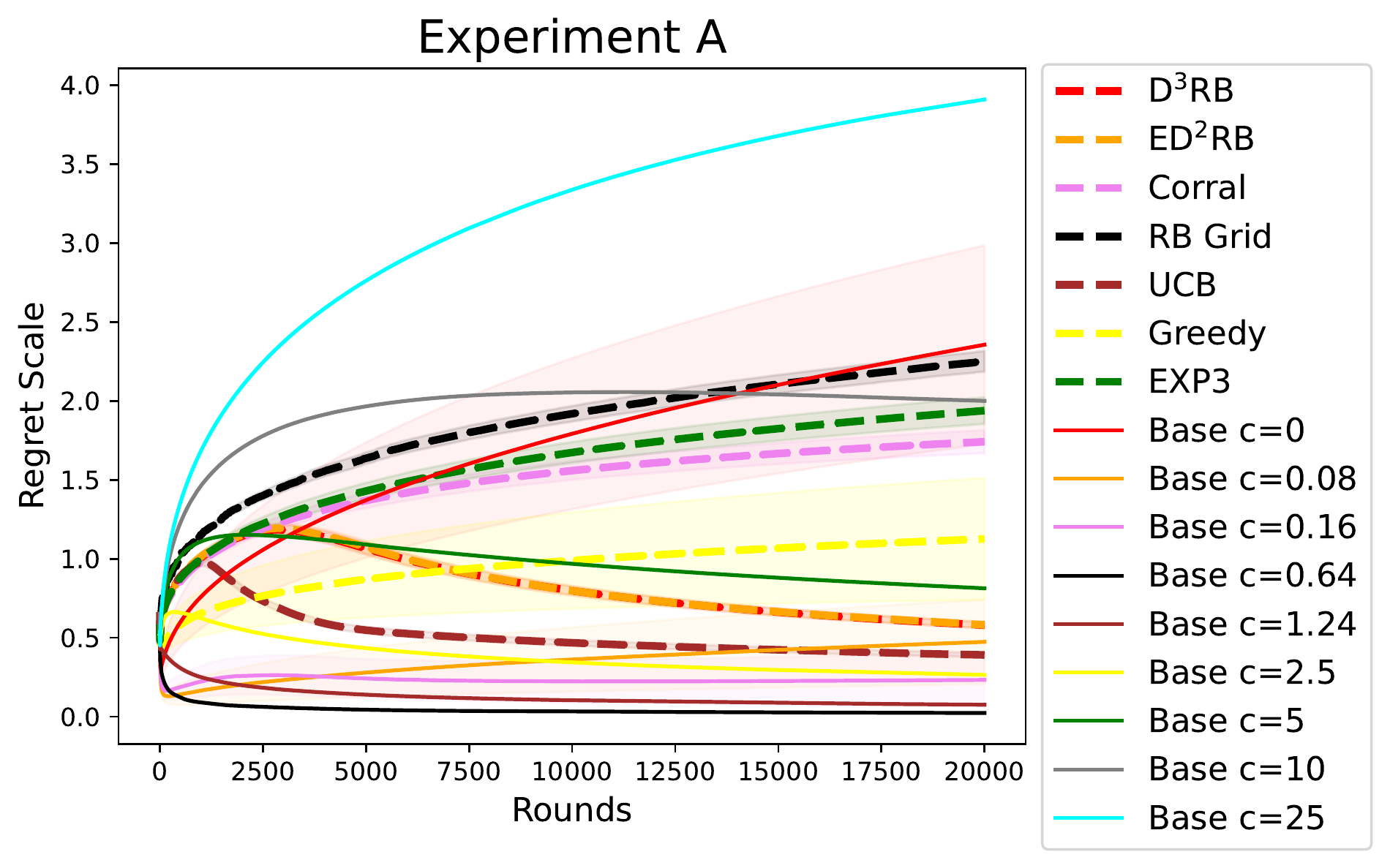}
        \end{minipage} \\
        \hline
        \begin{minipage}[t]{5cm}
        \vspace{-5.7cm}
\paragraph{Experiment B.} Selecting among different confidence scalings in a $2$ armed bandit problem with reward distributions $r \sim p_1$ and $r \sim p_2$ where samples from $p_i$ are of the form $30b_i$ where $b_i$ are two Bernoulli variables with means in $\{ .1, .2\}$. We test self-model selection among $10$ \textbf{UCB} base learners with confidence scalings $1$. 
        \end{minipage} &
        \begin{minipage}[t]{10cm}
            \includegraphics[width=\linewidth]{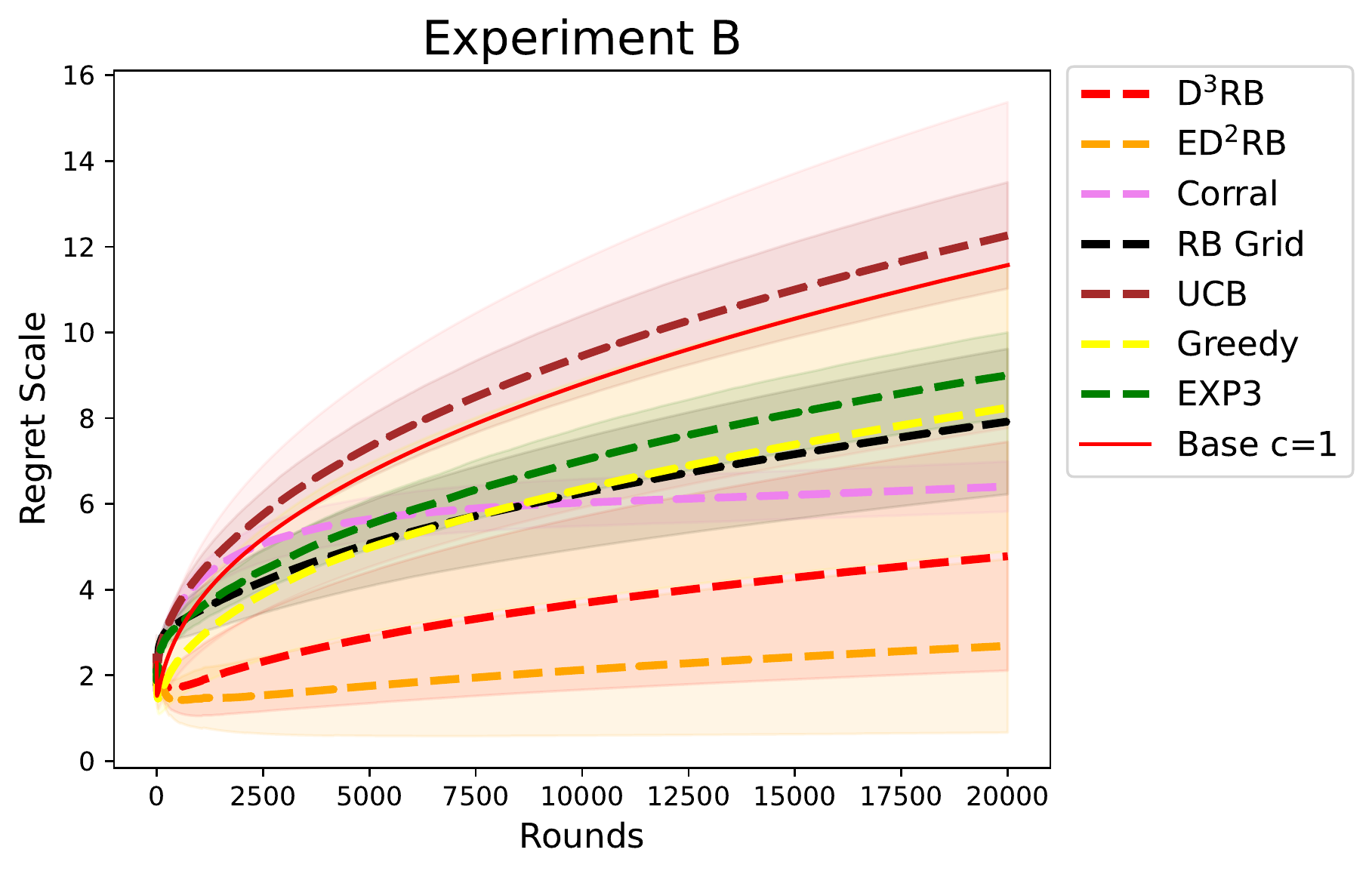}
        \end{minipage} \\
        \hline
        \begin{minipage}[t]{5cm}
        \vspace{-5.7cm}
\paragraph{Experiment C.} Linear bandit model selection where we select among LinTS base learners with different confidence scalings. The action set equals the $d = 5$ dimensional hypercube (i.e., the arm set equals $\{-1,1\}^d$ ). The $\theta_\star$ vector equals $(0,\ldots, 4)/\|(0,\ldots, 4)\| * 5$. The base learners are LinTS instances with confidence scalings in $\{ 0,.16, 2.5, 5, 25\}$. 
        \end{minipage} &
        \begin{minipage}[t]{10cm}
            \includegraphics[width=\linewidth]{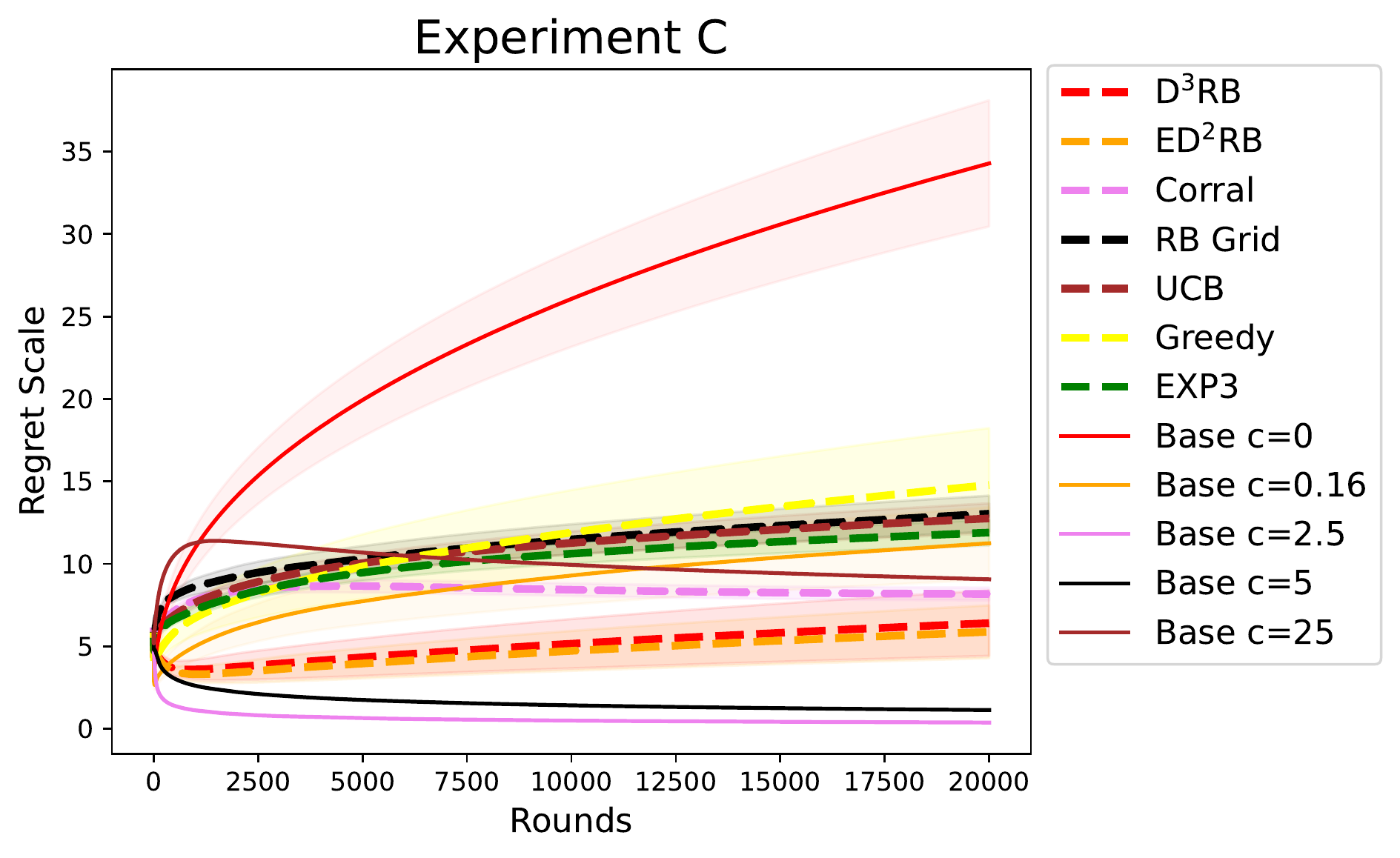}
        \end{minipage} \\
        \hline
    \end{tabular}
    \label{tab:figures3}
\end{table}

\begin{table}[htbp]
    \centering
    \begin{tabular}{|p{5cm}|p{10cm}|}
        \hline
        \textbf{Description} & \textbf{Figure} \\
        \hline
        \begin{minipage}[t]{5cm}
        \vspace{-5.7cm}
\paragraph{Experiment D.} Linear bandit model selection where we select among LinTS base learners with different confidence scalings. The action set equals the $d = 10$ dimensional hypercube (i.e., the arm set equals $\{-1/\sqrt{d},1/\sqrt{d}\}^d$ ). The $\theta_\star$ vector equals $(0,\ldots, 9)/\|(0,\ldots, 9)\| * 5$. The base learners are LinTS instances with confidence scalings in $\{ 0,.16, 2.5, 5, 25\}$. 
        \end{minipage} &
        \begin{minipage}[t]{10cm}
            \includegraphics[width=\linewidth]{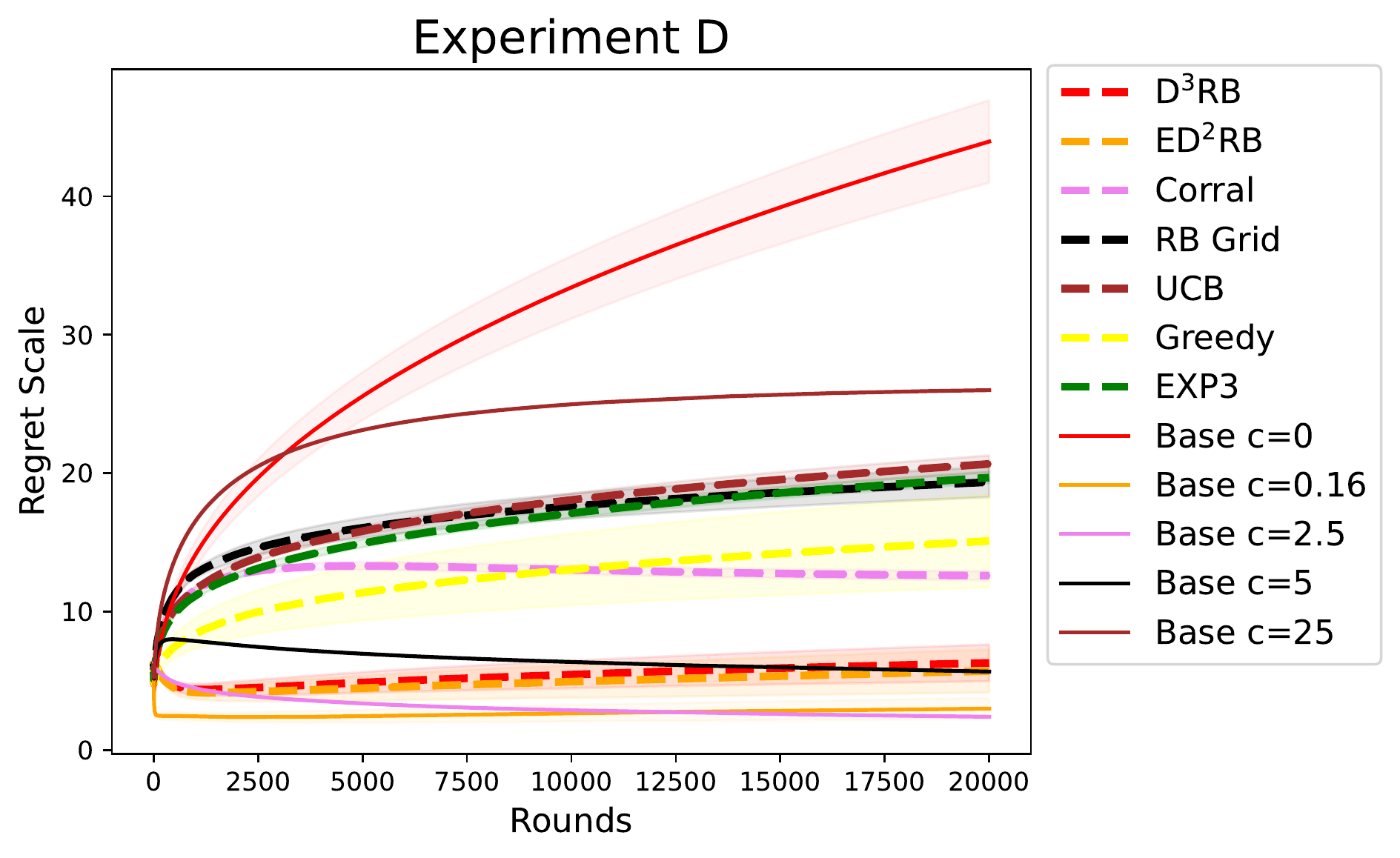}
        \end{minipage} \\
        \hline
        \begin{minipage}[t]{5cm}
        \vspace{-5.7cm}
\paragraph{Experiment E.} Linear bandit model selection where we select among LinTS base learners with different confidence scalings. The action set equals the $d = 100$ dimensional hypercube (i.e., the arm set equals $\{-1/\sqrt{d},1/\sqrt{d}\}^d$ ). The $\theta_\star$ vector equals $(0,\ldots, 99)/\|(0,\ldots, 99)\| * 5$. The base learners are LinTS instances with confidence scalings in $\{ 0,.16, 2.5, 5, 25\}$. 
        \end{minipage} &
        \begin{minipage}[t]{10cm}
            \includegraphics[width=\linewidth]{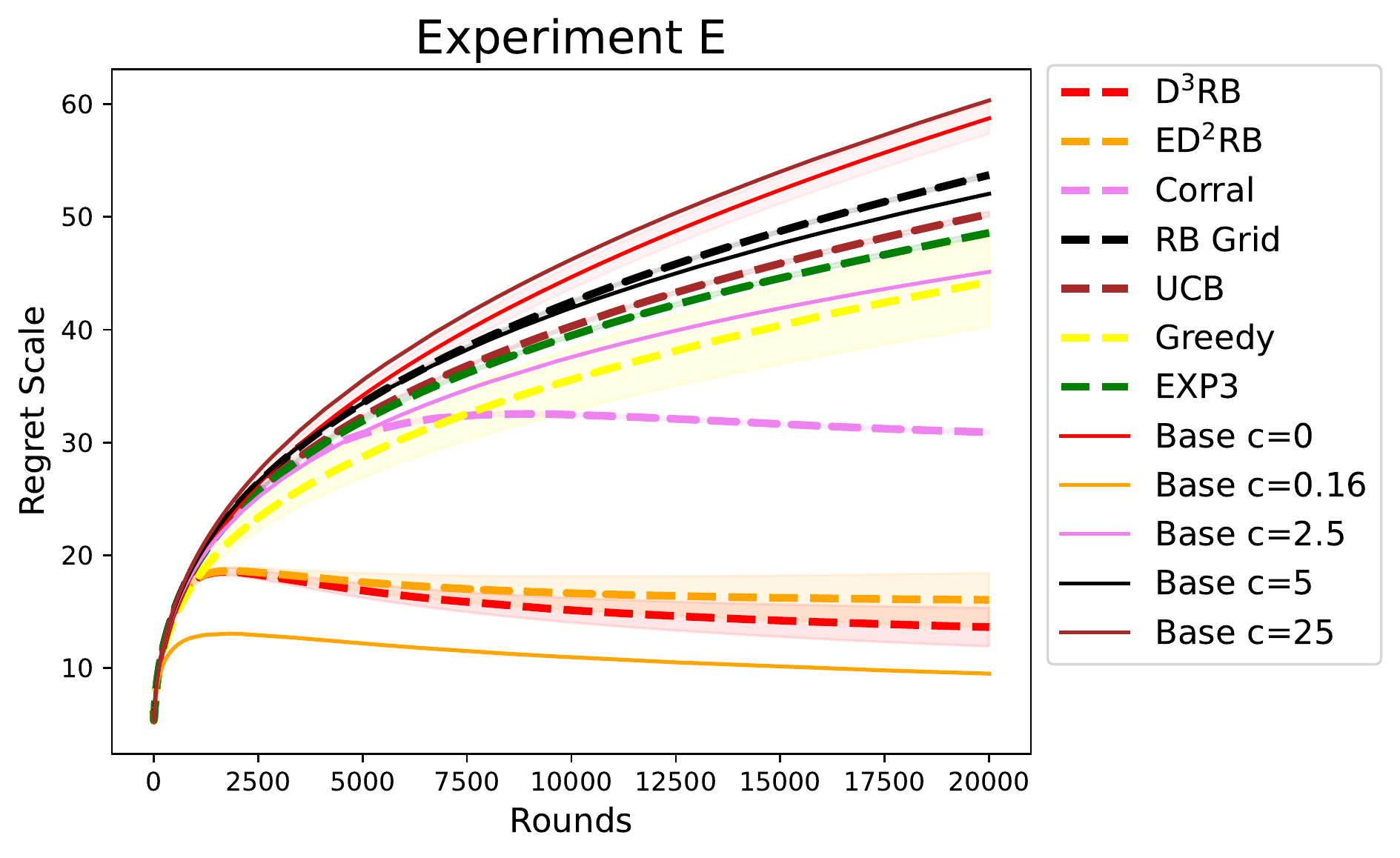}
        \end{minipage} \\
        \hline
        \begin{minipage}[t]{5cm}
        \vspace{-5.7cm}
\paragraph{Experiment F.}  Linear bandit model selection where we select among LinTS base learners with different confidence scalings. The action set equals the $d = 5$ dimensional unit sphere. The $\theta_\star$ vector equals $(0,\ldots, 4)/\|(0,\ldots, 4)\| * 5$. The base learners are LinTS instances with confidence scalings in $\{ 0,.16, 2.5, 5, 25\}$. 
        \end{minipage} &
        \begin{minipage}[t]{10cm}
            \includegraphics[width=\linewidth]{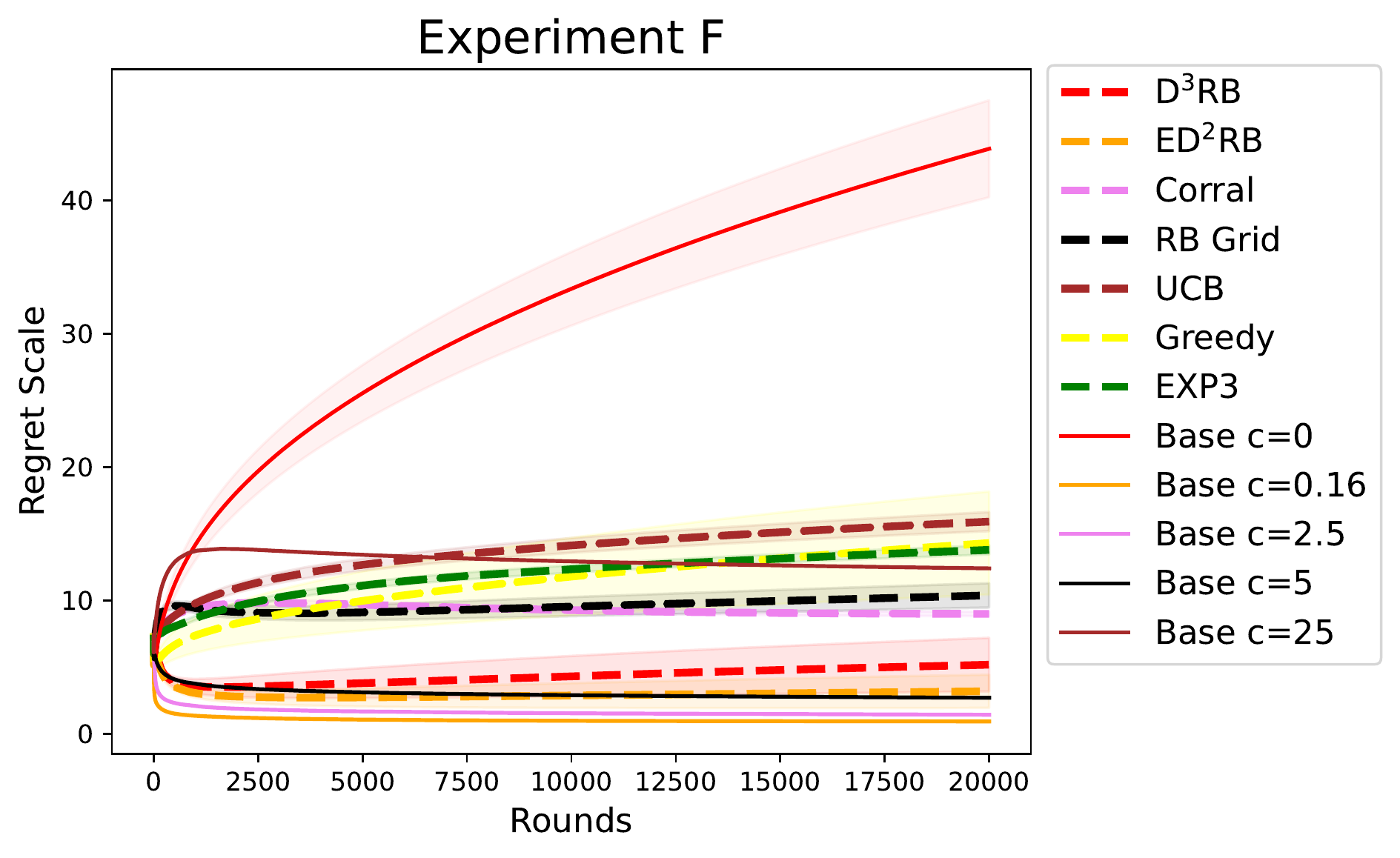}
        \end{minipage} \\
        \hline
    \end{tabular}
    \label{tab:figures4}
\end{table}

\begin{table}[htbp]
    \centering
    \begin{tabular}{|p{5cm}|p{10cm}|}
        \hline
        \textbf{Description} & \textbf{Figure} \\
        \hline
        \begin{minipage}[t]{5cm}
        \vspace{-5.7cm}
\paragraph{Experiment G.}  Linear bandit model selection where we select among LinTS base learners with different confidence scalings. The action set equals the $d = 100$ dimensional unit sphere. The $\theta_\star$ vector equals $(0,\ldots, 99)/\|(0,\ldots, 99)\| * 5$. The base learners are LinTS instances with confidence scalings in $\{ 0,.16, 2.5, 5, 25\}$. 
        \end{minipage} &
        \begin{minipage}[t]{10cm}
            \includegraphics[width=\linewidth]{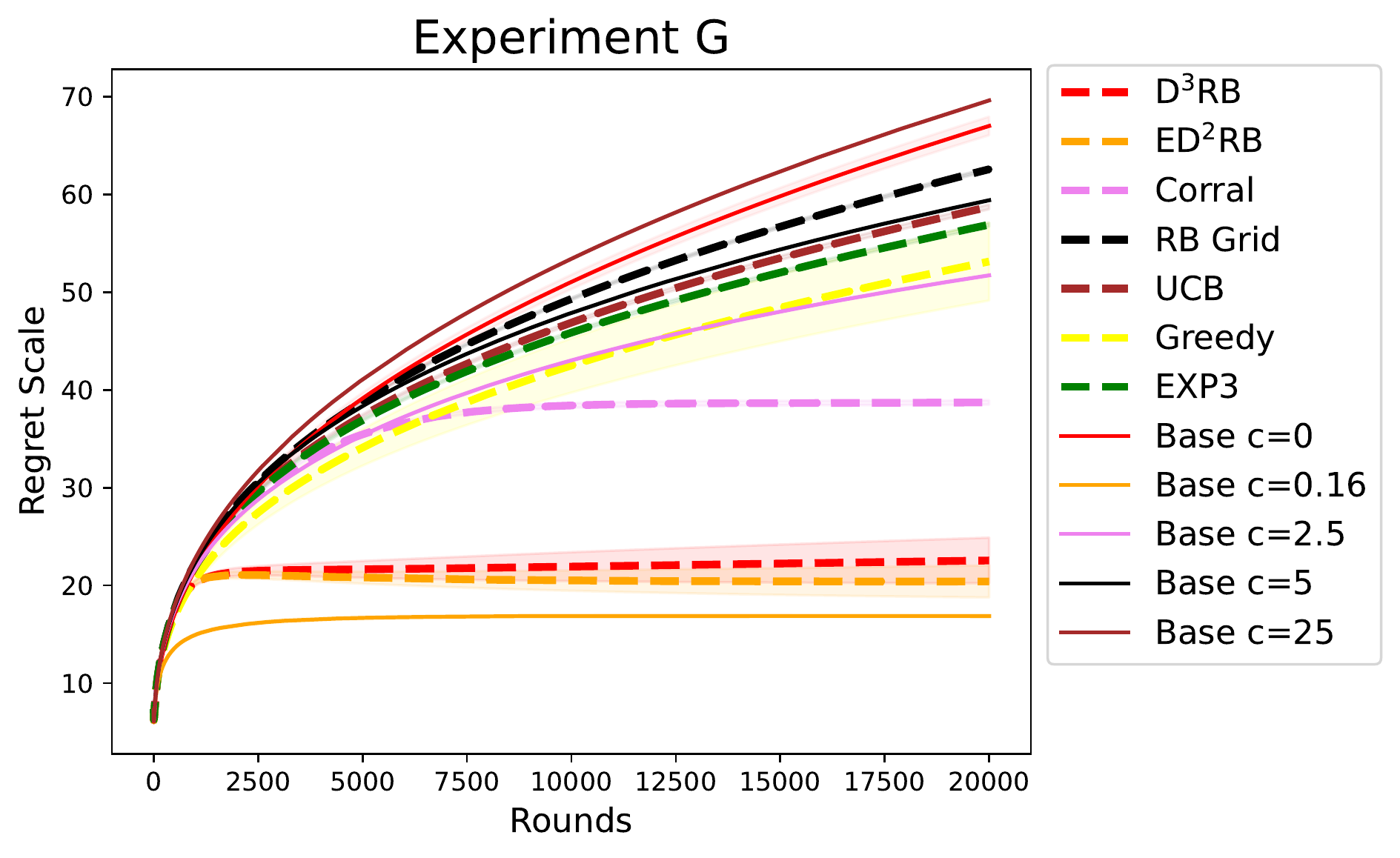}
        \end{minipage} \\
        \hline
        \begin{minipage}[t]{5cm}
        \vspace{-5.7cm}
\paragraph{Experiment H.} Contextual linear bandit model selection where we select among LinTS base learners with different confidence scalings. The contexts are generated by producing $10$ i.i.d. uniformly distributed vectors from the unit sphere. The ambient space dimension equals $d = 5$.  The $\theta_\star$ vector equals $(0,\ldots, 4)/\| (0,\ldots, 4)\|*5$. The base learners are LinTS instances with confidence scalings in $\{ 0,.16, 2.5, 5, 25\}$.
        \end{minipage} &
        \begin{minipage}[t]{10cm}
            \includegraphics[width=\linewidth]{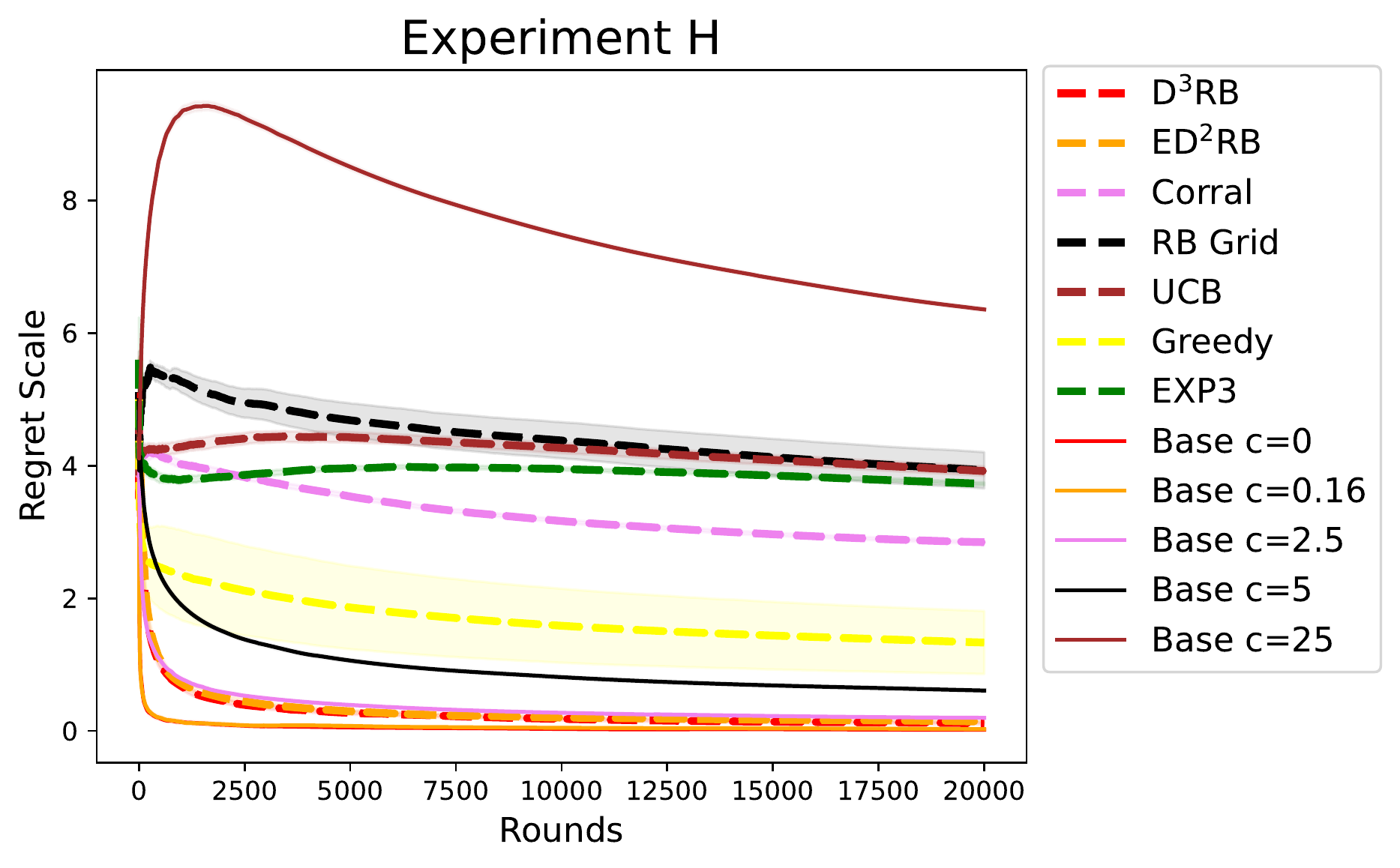}
        \end{minipage} \\
        \hline
        \begin{minipage}[t]{5cm}
        \vspace{-5.7cm}
\paragraph{Experiment I.} Contextual linear bandit model selection where we select among LinTS base learners with different confidence scalings. The contexts are generated by producing $10$ i.i.d. uniformly distributed vectors from the unit sphere. The ambient space dimension equals $d = 100$.  The $\theta_\star$ vector is $(0,\ldots, 99)/\| (0,\ldots, 99)\|*5$. The base learners are LinTS instances with confidence scalings in $\{ 0,.16, 2.5, 5, 25\}$.
        \end{minipage} &
        \begin{minipage}[t]{10cm}
            \includegraphics[width=\linewidth]{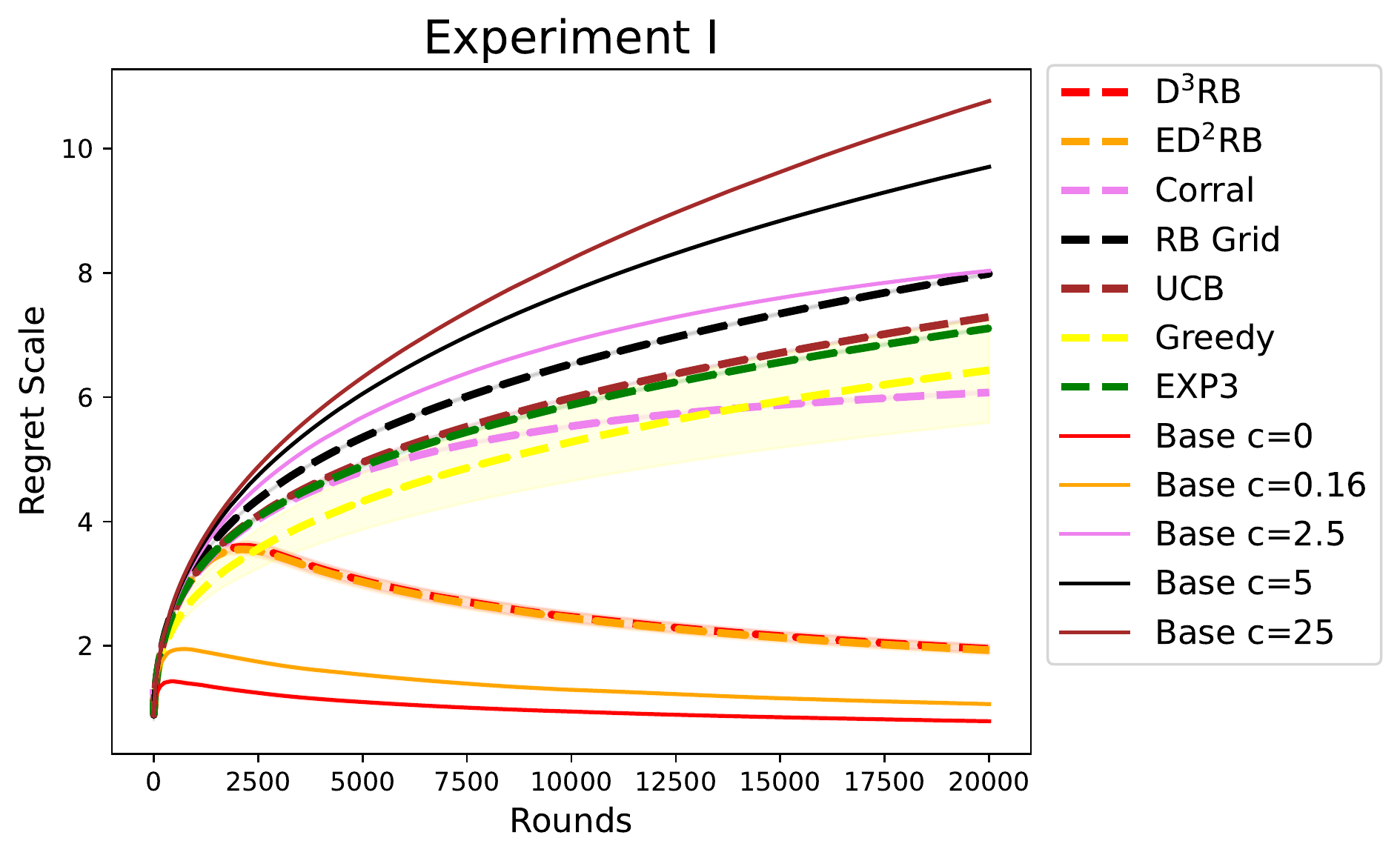}
        \end{minipage} \\
        \hline
    \end{tabular}
    \label{tab:figures5}
\end{table}

\begin{table}[htbp]
    \centering
    \begin{tabular}{|p{5cm}|p{10cm}|}
        \hline
        \textbf{Description} & \textbf{Figure} \\
        \hline
        \begin{minipage}[t]{5cm}
        \vspace{-5.7cm}
\paragraph{Experiment J.} Nested linear bandit model selection where we select among different LinTS base learners with different ambient dimensions. The action set is the unit sphere and the true ambient dimension is $30$. The $\theta_\star$ vector is $(0,1,\ldots,29)/\|(0,1,\ldots,29)\|*5$ and the base learners are LinTS instances with dimensions $d = 10, 30, 50, 100$,  and confidence scaling $2$. 
        \end{minipage} &
        \begin{minipage}[t]{10cm}
            \includegraphics[width=\linewidth]{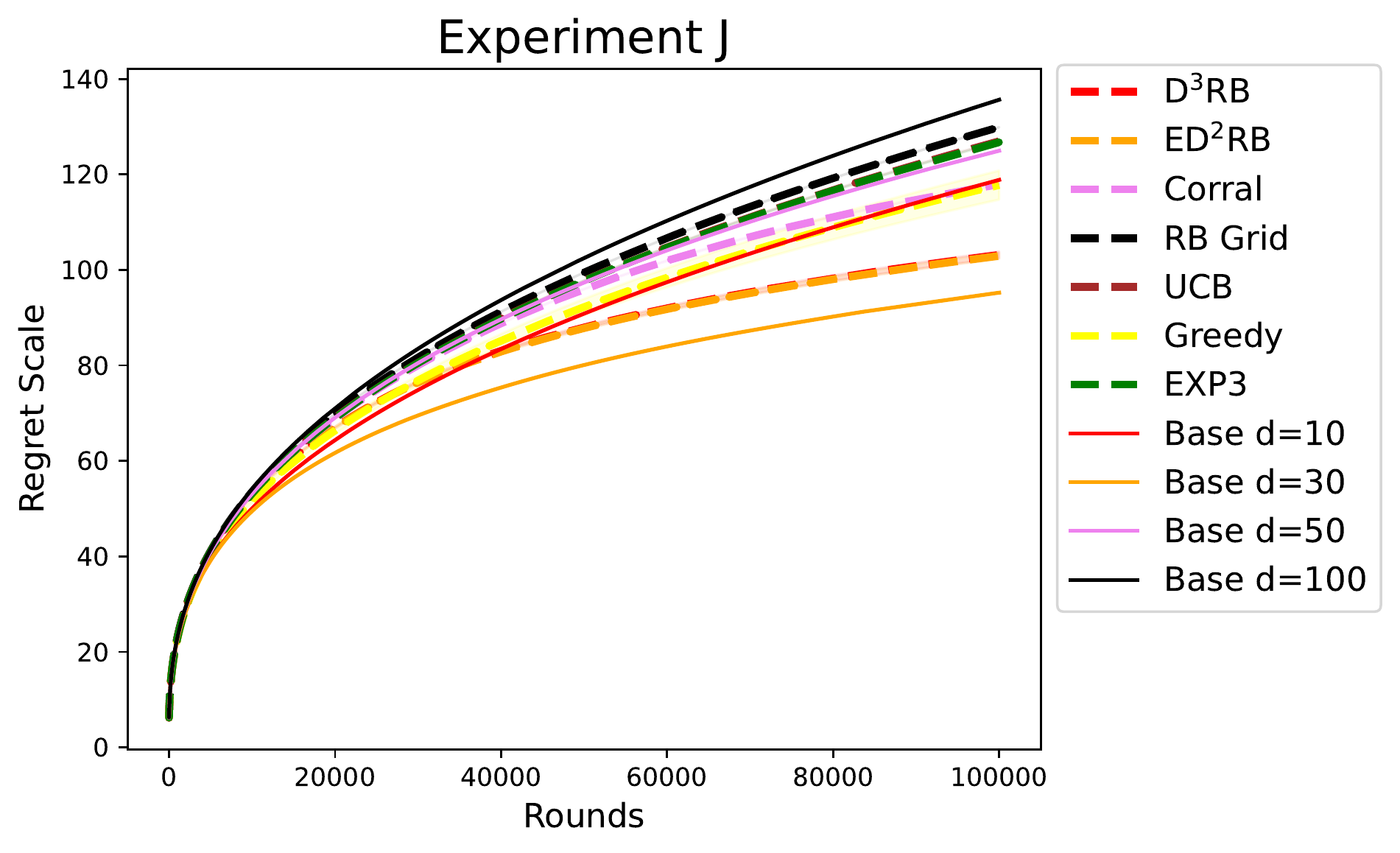}
        \end{minipage} \\
        \hline
        \begin{minipage}[t]{5cm}
        \vspace{-5.7cm}
\paragraph{Experiment K.} Nested linear bandit model selection where we select among different LinTS base learners with different ambient dimensions. The action set is the hypercube and the true ambient dimension is $5$. The $\theta_\star$ vector is $(0,1,2,3,4)$ and the base learners are LinTS instances with dimensions $d = 2, 5, 10, 15$ and confidence scaling $2$.  
        \end{minipage} &
        \begin{minipage}[t]{10cm}
            \includegraphics[width=\linewidth]{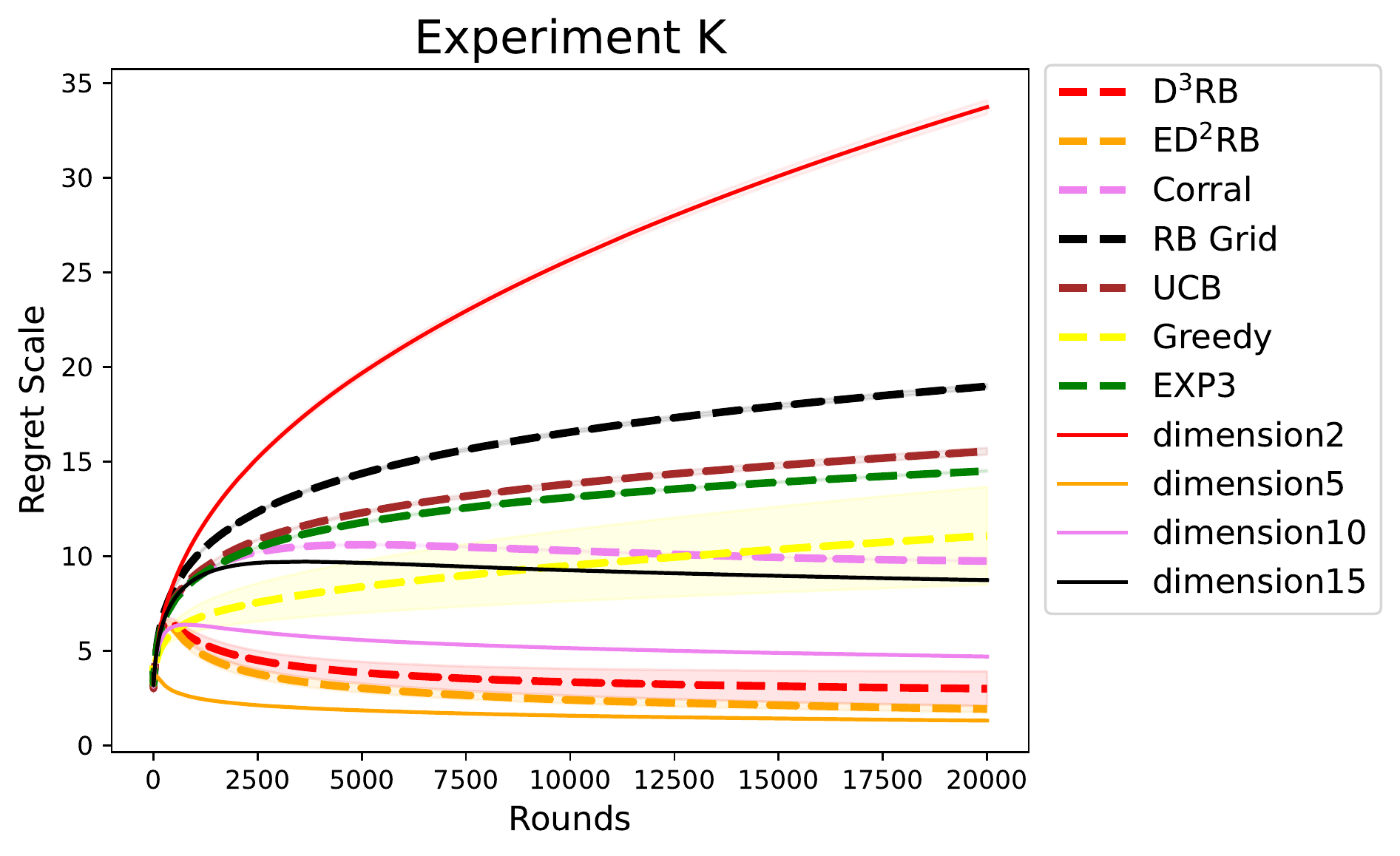}
        \end{minipage} \\
        \hline
        \begin{minipage}[t]{5cm}
        \vspace{-5.7cm}
\paragraph{Experiment L.} Nested linear bandit model selection where we select among different LinTS base learners with different ambient dimensions. The action set is the hypercube and the true ambient dimension is $5$. The $\theta_\star$ vector is $(0,1,\ldots,29)/\| (0,1,\ldots,29) \|*5$ and the base learners are LinTS instances with dimensions $d = 10,30, 50, 100$ and confidence scaling $2$.          
\end{minipage} &
        \begin{minipage}[t]{10cm}
            \includegraphics[width=\linewidth]{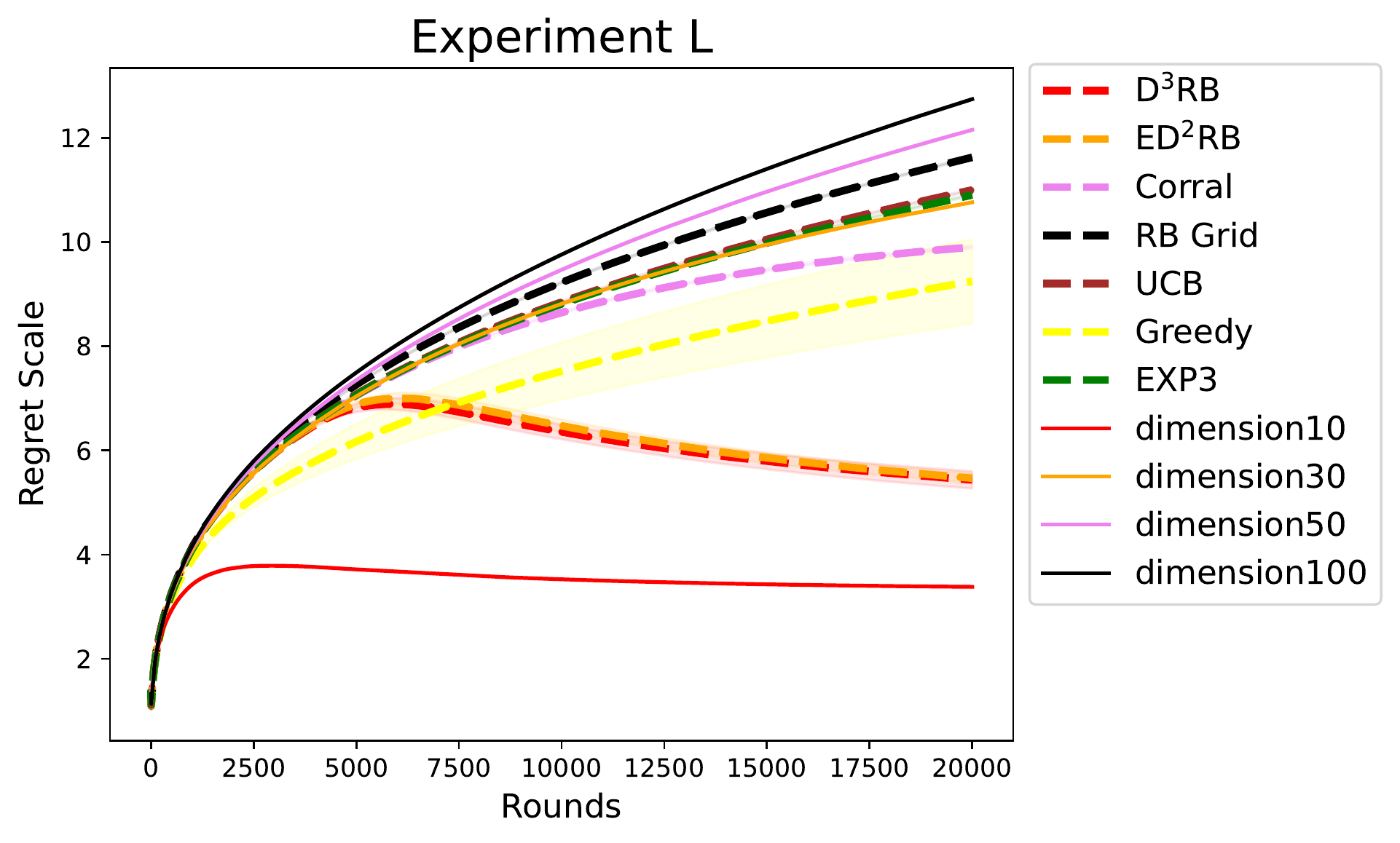}
        \end{minipage} \\
        \hline
    \end{tabular}
    \label{tab:figures6}
\end{table}

\begin{landscape}

\begin{table}[t]
\caption{Comparison of meta learners over the various experiments (Experiment 1 through 6, and Experiment A through L).
The lines corresponding to Experiments (``Name") 1--6 are also contained in the main body of the paper.
We report the cumulative regret (averaged over 100 repetitions $\pm 2 \times$standard error) at the end of the sequence of rounds. In bold is the best performer for each environment.
}
\label{tab:general_overview_appendix}
\bgroup
\def\arraystretch{1.1}
\begin{small}
\begin{tabular}{|lcccc | c| c | c | c | c | c | c |}
\hline
Name &Env. &Learners & Task & Arms & D$^3$RB &  ED$^2$RB & Corral & RB Grid &  UCB &Greedy & EXP3 \\
\hline
 \textcolor{red}{1.}& MAB &  UCB &  self &  Gaussian &  {$\bf 431 \pm 182 $}  &  {$560 \pm 240$} &  {$5498 \pm 340$}  &  {$6452 \pm 230$} & {$574 \pm 34$} &  {$6404 \pm 1102 $} & {$5892 \pm 356$} \\
 B.& MAB &  UCB &  self &  Bernoulli & {$ 6694 \pm 3738 $}  &  {$\bf 3765 \pm 2834$} &  {$8972 \pm 818$}  &  {$11093 \pm 2375$} & {$17174 \pm 1737$} &  {$11546 \pm 4698 $} & {$12606 \pm 1399$} \\
 \textcolor{red}{2.}& MAB &  UCB &  conf &  Gaussian & $1608 \pm 198$ &  $1413 \pm 208$ &  $2807 \pm 138$  &  $3452 \pm 110$ & $\bf 918 \pm 98$ &  $2505 \pm 362$ &  $3007 \pm 136$\\
 A.& MAB &  UCB &  conf &  Bernoulli & $ 811 \pm 27$ &  $813 \pm 28$ &  $2439 \pm 102$  &  $3152 \pm 91$ & $\bf 547 \pm 20$ &  $1576 \pm 541$ &  $2715 \pm 118$\\
 \textcolor{red}{3.}& LB &  LinTS &  conf  &  Sphere & $1150 \pm 134 $ &  $\bf 1135 \pm 148$ &  $2605 \pm 38$  &  $3169 \pm 66$ & $3052 \pm 36$ &  $2553 \pm 302$ &  $2491 \pm 36$\\
 F.& LB &  LinTS &  conf  &  Sphere & $7251 \pm 2820 $ &  $\bf 4458 \pm 1748$ &  $12594 \pm 237$  &  $14536 \pm 1278$ & $22286 \pm 1000$ &  $20034 \pm 5416$ &  $19316 \pm 488$\\
 G.& LB &  LinTS &  conf  &  Sphere & $31585 \pm 3269 $ &  $\bf 28579 \pm 2301$ &  $54230 \pm 318$  &  $87652 \pm 224$ & $82259 \pm 344$ &  $74433 \pm 5596$ &  $79706 \pm 277$\\
 C.& LB &  LinTS &  conf  &  Hypercube & $8971 \pm 2783 $ &  $\bf 8237 \pm 2246$ &  $11453 \pm 538$  &  $18264 \pm 1528$ & $17868 \pm 1277$ &  $20701 \pm 4835$ &  $16663 \pm 1091$\\
 D.& LB &  LinTS &  conf  &  Hypercube & $8831 \pm 1855 $ &  $\bf 8006 \pm 2174$ &  $17655 \pm 451$  &  $27119 \pm 1526$ & $28957 \pm 853$ &  $21167 \pm 4686$ &  $27576 \pm 582$\\
 E.& LB &  LinTS &  conf  & Hypercube &  $\bf 19088 \pm 2414 $ &  $ 22474 \pm 3343$ &  $43281 \pm 191$  &  $75233 \pm 255$ & $70389 \pm 318$ &  $61976 \pm 5567$ &  $68045 \pm 306$\\
 \textcolor{red}{4.}& CLB &  LinTS &  conf &  Context. & $411 \pm 100$ &  $\bf 406 \pm 94$ &  $1632 \pm 30$  &  $1073 \pm 184$ & $1644 \pm 160$&  $991 \pm 298$&  $1086 \pm 70$ \\
 H.& CLB &  LinTS &  conf & Context. &  $\bf 165 \pm 27$ &  $199 \pm 24$ &  $3986 \pm 31$  &  $5506 \pm 384$ & $5490 \pm 100$&  $1867 \pm 666$&  $5220 \pm 31$ \\
 I.& CLB &  LinTS &  conf & Context. & $2735 \pm 95$ &  $\bf 2705 \pm 101$ &  $8510 \pm 35$  &  $11184 \pm 26$ & $10210 \pm 25$&  $9012 \pm 1184$&  $9954 \pm 30$ \\
\textcolor{red}{5.}& LB &  LinTS &  dim & Sphere & $1733 \pm 230$ & $\bf 1556 \pm 198$ &  $3166 \pm 26$  &  $4223 \pm 40$ & $3932 \pm 16$ &  $3385 \pm 306$&  $3315 \pm 20$\\
J.& LB &  LinTS &  dim & Sphere & $\bf 93871 \pm 107$ & $ 93914 \pm 119$ &  $96194 \pm 39$  &  $97828 \pm 25$ & $96703 \pm 31$ &  $93681 \pm 926$&  $96659 \pm 27$\\
K.& LB &  LinTS &  dim & Hypercube & $4188 \pm 1288$ & $\bf 2681 \pm 250$ &  $13634 \pm 57$  &  $26570 \pm 174$ & $21765 \pm 255$ &  $15505 \pm 3622$&  $20308 \pm 69$\\
 L.& LB &  LinTS &  dim & Hypercube & $\bf 7614 \pm 233$ & $7660 \pm 126$ &  $13874 \pm 38$  &  $16287 \pm 19$ & $15410 \pm 20$ &  $12953 \pm 1132$&  $15275 \pm 22$\\
 \textcolor{red}{6.}& CLB &  LinTS &  dim & Context. & $\bf2347 \pm 102$ & $ 2365 \pm 96$ &  $5294 \pm 44$  &  $6258 \pm 38$ & $5718 \pm 50$ &  $4778 \pm 506$&  $5742 \pm 46$\\
 \hline
\end{tabular}
\end{small}
\egroup
\vspace{1mm}
\end{table}
\end{landscape}

\begin{table}[t]
\caption{Comparison of the \textbf{CorralLow}, \textbf{Corral} and \textbf{CorralHigh} meta learners over the various experiments (Experiment 1 through 6, and Experiment A through L). All of the experiments of this table were run for 20000 time-steps. This is in contrast with the results presented for Experiments 1-6 in \pref{tab:general_overview} and in \pref{tab:general_overview_appendix}. We report the cumulative regret (averaged over 100 repetitions $\pm 2 \times$standard error) at the end of the sequence of rounds. In bold is the best performer for each environment.
}
\label{tab:corral_overview_appendix}
\centering
\bgroup
\def\arraystretch{1.1}
\begin{tabular}{|lcccc | c| c | c | }
\hline
Name &Env. &Learners & Task & Arms & CorralLow &   Corral & CorralHigh \\
\hline
 \textcolor{red}{1.}& MAB &  UCB &  self &  Gaussian &  {$ 6609 \pm 433 $}  &  {$5498 \pm 340$} &  {$\bf 2598 \pm 256$}  \\
 B.& MAB &  UCB &  self &  Bernoulli & {$ 15324 \pm 1502 $}  &  {$\bf 8972 \pm 818$} &  {$10107 \pm 1516$}  \\
 \textcolor{red}{2.}& MAB &  UCB &  conf &  Gaussian & $5069 \pm 275$ &  $4670 \pm 251$ &  $\bf 3093 \pm 175$  \\
 A.& MAB &  UCB &  conf &  Bernoulli & $ 2742 \pm 113$ &  $2439 \pm 102$ &  $\bf 319 \pm 9$  \\
 \textcolor{red}{3.}& LB &  LinTS &  conf  &  Sphere & $32683 \pm 559 $ &  $\bf 19249 \pm 212$ &  $24287 \pm 850$ \\
 F.& LB &  LinTS &  conf  &  Sphere & $20073 \pm 752 $ &  $\bf 12594 \pm 237$ &  $15790 \pm 629$  \\
 G.& LB &  LinTS &  conf  &  Sphere & $80230 \pm 272 $ &  $54230 \pm 318$ &  $\bf 51499 \pm 1535$ \\
 C.& LB &  LinTS &  conf  &  Hypercube & $18783 \pm 911 $ &  $11453 \pm 538$ &  $\bf 4369 \pm 198$  \\
 D.& LB &  LinTS &  conf  &  Hypercube & $27575 \pm 730 $ &  $\bf 17655 \pm 451$ &  $20190 \pm 879$  \\
 E.& LB &  LinTS &  conf  & Hypercube &  $68764 \pm 349 $ &  $ 43281 \pm 191$ &  $\bf 41918 \pm 1236$  \\
 \textcolor{red}{4.}& CLB &  LinTS &  conf &  Context. & $11880 \pm 97$ &  $\bf 7606 \pm 93$ &  $10646 \pm 444$   \\
 H.& CLB &  LinTS &  conf & Context. &  $5269 \pm 31$ &  $\bf 3986 \pm 31$ &  $4624 \pm 160$   \\
 I.& CLB &  LinTS &  conf & Context. & $10045 \pm 26$ &  $\bf 8510 \pm 35$ &  $6823 \pm 171$  \\
\textcolor{red}{5.}& LB &  LinTS &  dim & Sphere & $44957 \pm 73$ & $\bf 22614 \pm 69$ &  $36384 \pm 926$ \\
J.& LB &  LinTS &  dim & Sphere & $ 96655 \pm 34$ & $ 96194 \pm 39$ &  $\bf 95394 \pm 153$  \\
K.& LB &  LinTS &  dim & Hypercube & $20533 \pm 86$ & $\bf 13634 \pm 57$ &  $14237 \pm 335$  \\
 L.& LB &  LinTS &  dim & Hypercube & $15312 \pm 23$ & $13874 \pm 38$ &  $\bf 12295 \pm 259$  \\
 \textcolor{red}{6.}& CLB &  LinTS &  dim & Context. & $\bf1295 \pm 50$ & $ 5294 \pm 44$ &  $4711 \pm 101$ \\
 \hline
\end{tabular}
\egroup
\vspace{1mm}
\end{table}

\begin{table}[t]
\caption{Comparison of the \textbf{EXP3Low}, \textbf{EXP3} and \textbf{EXP3High} meta learners over the various experiments (Experiment 1 through 6, and Experiment A through L). All of the experiments of this table were run for 20000 time-steps. This is in contrast with the results presented for Experiments 1-6 in \pref{tab:general_overview} and in \pref{tab:general_overview_appendix}. We report the cumulative regret (averaged over 100 repetitions $\pm 2 \times$standard error) at the end of the sequence of rounds. In bold is the best performer for each environment. All these algorithms have similar performance.
}
\label{tab:exp3_overview_appendix}
\centering
\bgroup
\def\arraystretch{1.1}
\begin{tabular}{|lcccc | c| c | c | }
\hline
Name &Env. &Learners & Task & Arms & EXP3Low &   EXP3 & EXP3High \\
\hline
 \textcolor{red}{1.}& MAB &  UCB &  self &  Gaussian &  {$ \bf 5733 \pm 376 $}  &  {$5892 \pm 356$} &  {$ 5743 \pm 354$}  \\
 B.& MAB &  UCB &  self &  Bernoulli & {$ 13764 \pm 1498 $}  &  {$\bf 12606 \pm 1399$} &  {$13308 \pm 1337$}  \\
 \textcolor{red}{2.}& MAB &  UCB &  conf &  Gaussian & $5224 \pm 299$ &  $5332 \pm 269$ &  $\bf 5136 \pm 279$  \\
 A.& MAB &  UCB &  conf &  Bernoulli & $ \bf 2607 \pm 109$ &  $2715 \pm 118$ &  $2613 \pm 113$  \\
 \textcolor{red}{3.}& LB &  LinTS &  conf  &  Sphere & $31546 \pm 356$ &  $\bf 31252 \pm 463$ &  $31256 \pm 441$ \\
 F.& LB &  LinTS &  conf  &  Sphere & $19453 \pm 535 $ &  $\bf 19316 \pm 488$ &  $19844 \pm 529$  \\
 G.& LB &  LinTS &  conf  &  Sphere & $79679 \pm 276 $ &  $79706 \pm 277$ &  $\bf 79311 \pm 290$ \\
 C.& LB &  LinTS &  conf  &  Hypercube & $17264 \pm 938 $ &  $\bf 16663 \pm 1091$ &  $ 17842 \pm 956$  \\
 D.& LB &  LinTS &  conf  &  Hypercube & $\bf 26789 \pm 570 $ &  $27576 \pm 582$ &  $27355 \pm 599$  \\
 E.& LB &  LinTS &  conf  & Hypercube &  $\bf 67479 \pm 325 $ &  $ 68045 \pm 306$ &  $ 67756 \pm 270$  \\
 \textcolor{red}{4.}& CLB &  LinTS &  conf &  Context. & $\bf 11877 \pm 207$ &  $ 11885 \pm 91$ &  $11921 \pm 71$   \\
 H.& CLB &  LinTS &  conf & Context. &  $5293 \pm 34$ &  $\bf 5220 \pm 31$ &  $5254 \pm 35$   \\
 I.& CLB &  LinTS &  conf & Context. & $\bf 9934 \pm 31$ &  $9954 \pm 30$ &  $9959 \pm 27$  \\
\textcolor{red}{5.}& LB &  LinTS &  dim & Sphere & $46136 \pm 55$ & $ 46129 \pm 52$ &  $\bf 46127 \pm 54$ \\
J.& LB &  LinTS &  dim & Sphere & $ 96663 \pm 26$ & $ 96659 \pm 27$ &  $\bf 96643 \pm 27$  \\
K.& LB &  LinTS &  dim & Hypercube & $\bf 20294 \pm 60$ & $ 20308 \pm 69$ &  $20315 \pm 86$  \\
 L.& LB &  LinTS &  dim & Hypercube & $15280 \pm 23$ & $15275 \pm 22$ &  $\bf 15256 \pm 22$  \\
 \textcolor{red}{6.}& CLB &  LinTS &  dim & Context. & $5761 \pm 59$ & $ 5742 \pm 46$ &  $\bf 5693 \pm 49$ \\
 \hline
\end{tabular}
\egroup
\vspace{1mm}
\end{table}

\end{document}